\documentclass[lettersize,journal]{IEEEtran}
\usepackage{amsmath,amsfonts}
\usepackage{algorithmic}
\usepackage{algorithm}
\usepackage{array}
\usepackage[caption=false,font=footnotesize,labelfont=rm,textfont=rm]{subfig}
\usepackage{textcomp}
\usepackage{stfloats}
\usepackage{url}
\usepackage{verbatim}
\usepackage{graphicx}
\usepackage{cite}

\usepackage{amsthm}
\usepackage{fdsymbol}
\usepackage{float}
\usepackage{tabularx}
\usepackage{multirow}
\usepackage{booktabs}

\usepackage[normalem]{ulem}

\newtheorem{lemma}{Lemma}
\newtheorem{theorem}{Theorem}
\newtheorem{corollary}{Corollary}
\newtheorem{proposition}{Proposition}

\labelformat{equation}{(#1)}

\begin{document}

\title{On the Role of Entropy-Based Loss for Learning Causal Structure with Continuous Optimization}

\author{Weilin~Chen,
        Jie~Qiao,
        Ruichu~Cai\textsuperscript{*},~\IEEEmembership{Senoir Member,~IEEE,}
        and~Zhifeng~Hao,~\IEEEmembership{Member,~IEEE,}
\thanks{Manuscript received XXX; revised XXX. This research was supported in part by National Key R\&D Program of China (2021ZD0111501), National Science Fund for Excellent Young Scholars (62122022), Natural Science Foundation of China (61876043, 61976052, 62206064, 62206061), the major key project of PCL (PCL2021A12). We appreciate the comments from anonymous reviewers, which greatly helped to improve the paper. \textit{(Corresponding author: Ruichu Cai). The first two authors contributed equally to this work.}}
\thanks{W. Chen is with the School of Computer, Guangdong University of Technology, Guangzhou, 510006, China (e-mail: chenweilin.chn@gmail.com).}
\thanks{J. Qiao and R. Cai are with the School of Computer Science, Guangdong University of Technology, Guangzhou, 510006, China and Peng Cheng Laboratory, Shenzhen, China (e-mail: qiaojie.chn@gmail.com; cairuichu@gmail.com).}
\thanks{Z. Hao is with College of Science, Shantou University, Guangdong, 515063, China (e-mail: haozhifeng@stu.edu.cn).}
}

\markboth{Journal of \LaTeX\ Class Files,~Vol.~14, No.~8, August~2021}%
{Shell \MakeLowercase{\textit{et al.}}: A Sample Article Using IEEEtran.cls for IEEE Journals}


\maketitle


\begin{abstract}
  Causal discovery from observational data is an important but challenging task in many scientific fields. 
  A recent line of work formulates the structure learning problem as a continuous constrained optimization task using an algebraic characterization of directed acyclic graphs and the least-square loss function. 
  Though the least-square loss function is well justified under the standard Gaussian noise assumption, it is limited if the assumption does not hold. In this work, we theoretically show that the violation of the Gaussian noise assumption will hinder the causal direction identification, making the causal orientation fully determined by the causal strength as well as the variances of noises in the linear case and by the strong non-Gaussian noises in the nonlinear case. Consequently, we propose a more general entropy-based loss that is theoretically consistent with the likelihood score under any noise distribution. We run extensive empirical evaluations on both synthetic data and real-world data to validate the effectiveness of the proposed method and show that our method achieves the best in Structure Hamming Distance, False Discovery Rate, and True Positive Rate matrices.
\end{abstract}


\section{Introduction}\label{sec:introduction}

\IEEEPARstart{L}{earning} causal structure from observational data has become an important topic in many scientific fields, such as economics \cite{ghysels2016testing}, biology \cite{grosse2016identification}, neuroinformatics\cite{9411670}, and social science \cite{cai2016understanding}. Due to the expensive cost or the ethic of randomized experiments, the task of causal discovery from purely observational data has drawn much attention. 

Many approaches have been proposed for learning causal structure. Traditionally, by utilizing the conditional independence property among observed variables, the constraint-based approaches have been proposed, e.g., PC algorithm \cite{spirtes2000causation}, but only identify the underlying directed acyclic graph (DAG) up to Markov equivalence class \cite{andersson1997characterization}. Alternatively, by introducing a certain class of Structure Causal Model (SCM), and further assuming the causal mechanism that the noise and the hypothetical cause are independent, the functional-based causal models have been proposed, e.g., the Linear Non-Gaussian Acyclic Model (LiNGAM) \cite{shimizu2006linear}, the Additive Noise Model (ANM) \cite{hoyer2008nonlinear}, the Post-Nonlinear (PNL) causal Model \cite{zhang2012identifiability}. However, due to the intractable search space superexponential in the number of graph nodes, learning DAGs using functional-based causal models is challenging.

Recently, Zheng et al. \cite{zheng2018dags} propose a  method named NOTEARS: Non-combinatorial Optimization via Trace Exponential and Augmented lagRangian for Structure learning, which formulates the causal discovery problem as a continuous optimization problem using least-square loss with a DAG constraint under the additive noise model assumption.
Such a technique has been extensively developed and applied to learning linear or nonlinear causal structures. Yu et al. \cite{yu2019dag} introduce a variational autoencoder framework for modeling the generative process of a causal structure equipped with evidence lower bound with a Gaussian prior of noise which is implemented by least-square loss. Ng et al. \cite{ng2019graph} and zheng et al. \cite{zheng2020learning} extend a linear causal model into a nonlinear causal model using neural networks but still rely on the least-square loss for reconstruction. 

However, our analysis shows that using least-square loss as the score function will hinder the causal direction identification. In a linear system, such loss will be highly influenced by the causal strengths and the noise variances leading to a bias estimation compared with the likelihood score. In a more general nonlinear system, the causal direction will be incorrectly identified if noise distribution has a strong non-Gaussianity property.

In this paper, we show that the entropy-based loss is consistent with the likelihood score under the additive noise model using any noise distribution, and thus we advocate using the entropy-based loss instead of least-square loss. Overall, our contributions are as follows: 
\begin{enumerate}
	\item In Section \ref{sec:Limitations}, we study the limitations of least-square loss and provide theoretical bounds that least-square loss will fail to identify the causal direction in a linear and nonlinear system, respectively.
	\item In Section \ref{sec:entropy}, we build a connection between the entropy-based loss and the log-likelihood score and further provide the theory for its validity.
	\item In Section \ref{sec:experiment}, we run extensive experiments over linear and nonlinear systems by using both synthetic data and real-world data and show that our method performs more effectively and stably.
\end{enumerate}

\section{Related Works}

\textbf{Related to least-square loss:} 
Bl{\"o}baum et al. \cite{blobaum2019analysis} address the problem of inferring the causal relation between two variables using least-square loss under the modularity property, i.e., the independence between the function and the distribution of cause. Peters et al. \cite{peters2014identifiability} prove the full identifiability of the linear Gaussian structural equation in the case that all noise variables have the same variance, which can be recovered by the least-square loss in our analysis. Loh et al. \cite{lincausalbyinverse} prove the identification of a linear Gaussian structural causal model with known noise variances. 

\textbf{Related to causal discovery:} One of the categories of causal discovery approaches is to learn causal structures by iteratively conducting independence tests, such as IC \cite{pearl1995theory}, PC \cite{spirtes2000causation}, M2LC \cite{9540220}, M3B \cite{8361079}. Another category of approaches is based on causal mechanism assumptions. Shimizu et al. \cite{shimizu2006linear} propose LiNGAM by assuming linearity and non-Gaussianity. Many variants improve the efficiency \cite{shimizu2011directlingam, xie2019efficient}, extend to time-series data \cite{hyvarinen2008causal}, or relax the causal sufficiency assumption \cite{chen2021causal}. Hoyer et al. \cite{hoyer2008nonlinear} propose ANM with additive noise assumption. Peters et al. \cite{peters2010identifying} study the identifiability of discrete ANM, and Wei et al. \cite{weicausal} extend the orientation rule of ANM by Wasserstein distance. Zhang et al. \cite{zhang2010distinguishing, zhang2012identifiability} propose PNL causal model with inner additive noise and give its identifiability condition. B{\"u}hlmann et al. \cite{buhlmann2014cam} propose the causal additive model (CAM) by assuming the structural equations are additive. 
Recently, based on the acyclic constraint, the gradient-based approaches has drawn attention.
Zheng et al. \cite{zheng2018dags} propose a DAG constraint named NOTEARS for learning causal structure with continuous optimization. Ng et al. \cite{ng2020role} apply soft sparsity and DAG constraints to learn linear DAGs based on a likelihood-based score function. Yu et al. \cite{yu2019dag} propose an alternative characterization of acyclicity and utilize a generative model to learn the nonlinear causal structure. Ng et al. \cite{ng2019graph} utilize a graph autoencoder framework to extend the linear SCM to nonlinear SCM. Ng et al. \cite{ng2019masked} extend to nonlinearity by putting the weighted matrix on the first layer of MLP. Lachapelle et al. \cite{lachapelle2019gradient} express weighted matrix by neural network paths to deal with the nonlinear case. Zheng et al. \cite{zheng2020learning} propose a more general acyclicity constraint based on partial derivatives to support nonlinear models. Liang et al. \cite{10032707} propose the local structures learning algorithm GraN-LCS based on the DAG constraint. Wren et al.\cite{wren2022learning} learn discrete DAG using such a constraint. Moreover, many methods combine NOTEARS with other domains. Zhu et al. \cite{zhu2019causal} use Reinforcement Learning to recover the causal structure with the best scoring under the acyclicity constraint. Pamfil et al. \cite{pamfil2020dynotears} extend to recovering the causal structure on time-series data. Brouillard et al. \cite{brouillard2020differentiable} make use of interventional data for differentiable causal discovery. Wehenkel et al. \cite{wehenkel2021graphical} utilize the normalizing flow to learn the causal mechanism with the DAG constraint. Zeng et al. \cite{ijcai2021-289} extend to learning the causal structure on multi-domain data. Faria et al.\cite{pmlr-v177-faria22a} learn causal structures under latent intervention based on DAG constraint. Ng et al. \cite{ng2022towards} propose a federated Bayesian network structure learning method using such a DAG constraint. Yang et al. \cite{yang2021causalvae} and Mao et al. \cite{mao2022towards} apply such a DAG constraint to learn the causal representation. Zhang et al.  \cite{zhang2022nonparametric} apply DAG constraint to learn a forest-structured neural topic model to capture relationships between topics.

Our work is different from the works above. Many mentioned works \cite{zheng2018dags, ng2020role, yu2019dag, ng2019graph, ng2019masked, lachapelle2019gradient,pamfil2020dynotears, ng2022towards} are based on the least-square loss to orient causal direction. However, our analysis shows that it has no identifiability for causal discovery (see Section \ref{sec:Limitations}), and we propose a better entropy-based loss that identifies causal direction correctly (see Section \ref{sec:entropy}).

\section{Preliminary}

\subsection{Problem Formulation}
We tackle the problems of multivariate causal discovery, aiming to learn a directed acyclic graph (DAG) $\mathcal{G}$ that represents a joint distribution $p(\Vec{X})$ over a set of random variables $\Vec{X}=\{X_1, X_2,\dots, X_d\}$. A directed arrow from $X_i$ to $X_j$ represents the cause-effect relation between two variables, where we say $X_i$ is the parent of $X_j$. Let $X_{pa\left(i\right)}$ denote all direct parents of $X_i$ with respect to a DAG $\mathcal{G}$. The dependence relation between $X_i$ and $X_{pa\left(i\right)}$ can be represented by a conditional distribution $p(X_i|X_{pa\left(i\right)})$. Given a set of observational data $\left\{ x_1^{(i)},\dots,x_d^{(i)}\right\}_{i=1}^m $ which is sampled from a joint distribution $p(\Vec{X})$,  we assume that the joint distribution is Markov with respect to a ground truth DAG $\mathcal{G}$ and can be factorized as $p(\Vec{X})=\prod_{i=1}^dp(X_i|X_{pa\left(i\right)})$. Then the corresponding design matrix is $\mathbf{X} \in \mathbb{R}^{m\times d}$. In this paper, we focus on the task of learning the underlying DAG given observational data. The mathematical notations and the corresponding descriptions
are summarized in Table \ref{tab:notation} .

\begin{table}[!t]
\renewcommand{\arraystretch}{1.0}
\caption{Mathematical notation and descriptions}
\label{tab:notation}
\centering
\resizebox{\linewidth}{!}{
\begin{tabular}{l|l}
\hline
Notation & Description \\
\hline
$\mathcal{G}$ & Directed acyclic graph (DAG) \\
\hline
$\Vec{X}$ &  Random variables vector $\Vec{X}=\{X_1, X_2,\dots,X_d\}$  \\
\hline
$X_i$ & The $i$th variable (element) of $\Vec{X}$ \\
\hline
$\Vec{N}$ &  Random variables vector $\Vec{N}=\{N_1, N_2,\dots,N_d\}$  \\
\hline
$N_{i}$ & The noise variable of $X_i$ \\
\hline
$X_{pa\left(i\right)}$ & The parent variable of $X_i$ \\
\hline
$f_{i}$ & The mapping from $X_{pa\left(i\right)}$  to $X_i$ \\
\hline
$\mathbf{X}$ & The design matrix, $\mathbf{X} \in \mathbb{R}^{m\times d}$ \\
\hline
$X$ & The cause variable in pairwise $X$, $Y$   \\
\hline
$Y$ & The effect variable in pairwise $X$, $Y$   \\
\hline
$N_X$ &  The noise of cause variable  \\
\hline
$N_Y$ &  The noise of effect variable  \\
\hline
$f$ &  The nonlinear mapping from $X$ to $Y$ in causal direction \\
\hline
$g$ &  The nonlinear mapping from $Y$ to $X$ in anti-causal direction  \\
\hline
$\hat{N}_X$ &  The residual of $X$ denoted in anti-causal direction,\\
 & $\hat{N}_X := X-g(Y) $  or $\hat{N}_X := X- \hat{a}Y $ \\ 
\hline
$\hat{N}_Y$ &  The residual of $Y$ denoted in anti-causal direction, $\hat{N}_Y := Y$  \\
\hline
$p$ & Probability density function  \\
\hline
$q$ & The probability density function of standard Gaussian distribution  \\
\hline
$I(X,Y)$ & The mutual information between variable $X$ and $Y$ \\
\hline
$H(X)$ & The entropy of variable $X$ \\
\hline
$\sigma_X^2$ & The variance of variable $X$ \\
\hline
\end{tabular}
}
\end{table}

\subsection{Additive Noise Model}
Causal discovery problems can be formalized by a Structure Causal Model (SCM) \cite{didelez2001judea}. Given a set of random variables $ \Vec{X}=\left\{X_1,X_2,\dots,X_d\right\}$ and the corresponding noises $\Vec{N}=\left\{N_1,N_2,\dots,N_d\right\}$, SCM between $X_{i}$ and its direct parents $X_{pa\left(i\right)}$ with respect to a DAG $\mathcal{G}$ is defined as $X_i = f_i \left(X_{pa\left(i\right)}, N_i\right) $, where $f_i$ could be a linear or nonlinear function. Following the previous works \cite{zheng2018dags}, we assume the causal sufficiency and consider the Additive Noise Model (ANM): 
\begin{equation*}
    X_i = f_i \left( X_{pa\left(i\right)} \right) + N_i,
\end{equation*}
where $N_i$ is the exogenous additive noise and $N_i \Vbar X_{pa\left(i\right)}$.
The identifiability of ANM depends on the asymmetry independence property, i.e., independence between the causes and the noise but not vice versa. Such an asymmetry greatly motivates our entropy-based loss for causal orientation in Section \ref{sec:entropy}.

\subsection{NOTEARS}

 Zheng et al. \cite{zheng2018dags} proposed a continuous optimization for learning causal structure, using least-square loss with an acyclicity constraint.
In particular, the directed graph in a linear SCM can be encoded by a weighted adjacency matrix $W \in \mathbb{R}^{d \times d}$, i.e., $ \Vec{X} =   W^{\mathsf{T}}\Vec{X} + \vec{N}$. NOTEARS shows that $W$ represents a DAG if and only if $tr\left(e^{W \circ W}\right)-d = 0$ holds, where $ \circ $ denotes the Hadamard product. Then, NOTEARS formulates the causal structure learning problem as the following continuous optimization problem:
\begin{equation}
  \begin{aligned}
  	& \mathop{\arg\min}_{W \in \mathbb{R}^{d \times d} } \frac{1}{2m} \| \mathbf{X} - \mathbf{X} W \|^2_F + \lambda \left\|W\right\|_1 \\
  	& \text{subject to} \quad tr\left(e^{W \circ W}\right)-d = 0,
  \end{aligned}
\end{equation}
where $ \frac{1}{2m} \| \mathbf{X} - \mathbf{X} W \|^2_F$ is the least-square loss and is equal, up to constant, to the log-likelihood score of a linear Gaussian DAG with equal noise variances, and $\left\|W\right\|_1$ denotes the $\ell_1$ penalty term on the causal structure.

\subsection{Differential Entropy Estimator}

To estimate the entropy in continuous optimization, we use the estimator proposed by Hyv{\"a}rinen \cite{10.5555/302528.302606}. In detail, let $H\left(X\right) = - \int p\left(X\right)\log p\left(X\right) dX$ denote the entropy of variable $X$. It gives a way to estimate entropy as follows:
\begin{equation}
	\begin{aligned}
		H(X) \approx  H(\nu) -\left[  k_{1}\left(E\left\{\bar{G}_{1}(X)\right\}\right)^{2} 
		\right. \\ \phantom{=\;\;}\left.
		+k_{2}\left(E\left\{\bar{G}_{2}(X)\right\}-E\left\{\bar{G}_{2}(\nu)\right\}\right)^{2}\right],
	\end{aligned}
\end{equation}
where $k_{1} = 36/\left(8\sqrt{3}-9\right)$, $k_{2} = 24/\left(16\sqrt{3}-27\right)$, $\bar{G}_{1}(X) = X\exp\left(-X^{2}/2\right)$, $\bar{G}_{2}(X) = \exp\left(-X^{2}/2\right)$, $\bar{G}_{2}(\nu) = \sqrt{1/2}$, and $H(\nu) = \frac{1}{2}\left(1+\log(2\pi)\right)$. For more details, see \cite{10.5555/302528.302606}.

\section{Limitations of Least-Square Loss} \label{sec:Limitations}

\begin{figure*}[t] 
	\centering
	\subfloat[Loss difference varies with various causal strength]{ 
		\includegraphics[width=0.3\textwidth]{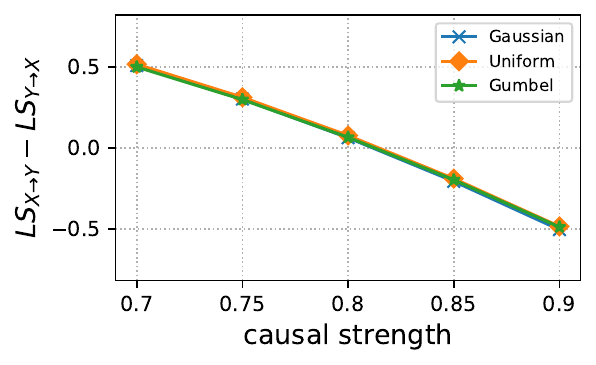} 
		\label{figure 1a}
	}
	\hfil
	\subfloat[Loss difference varies with various variance of $N_{X}$]{
		\includegraphics[width=0.3\textwidth]{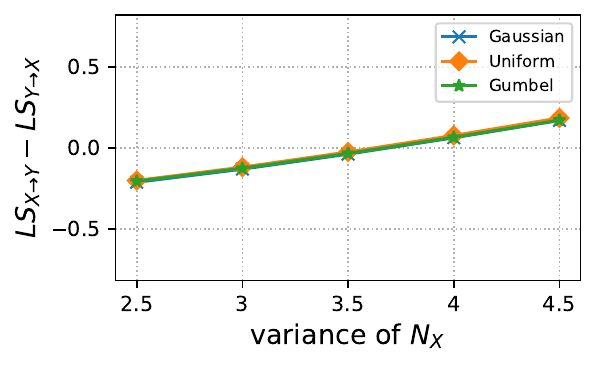}
		\label{figure 1b}
	}
	\hfil
	\subfloat[Loss difference varies with various variance of $N_{Y}$]{
		\includegraphics[width=0.3\textwidth]{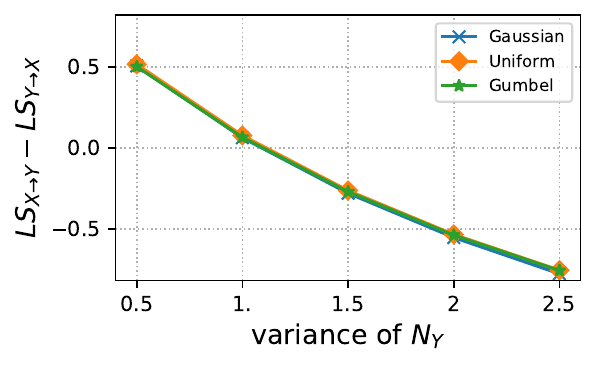}
		\label{figure 1c}
	}
	
	\caption{Control experiments of loss difference, i.e, $LS_{X \to Y} - LS_{Y\to X}$, with respect to different conditions in the linear system. At each experiment, we only vary one of the parameters while fixing the others. The default parameters are as follows: $\alpha=0.85$, $\sigma_{N_{X}}^2=4$, and $\sigma_{N_{Y}}^2 = 1$. The controlled parameters will vary in the range as follows: $\alpha \in [0.7,0.9]$, $\sigma_{N_{X}}^2 \in [2.5,4.5]$, $\sigma_{N_{Y}}^2\in [0.5,2.5]$.}
	 \label{figure 1} 
\end{figure*}

This section discusses why the least-square loss based methods fail to discover the underlying causal structure, and we theoretically show the conditions that they fail in both linear and nonlinear cases.

Without loss of generality, we analyze the causal pair of two variables and assume that:
\begin{equation} \label{eq:causal anm}
	\begin{aligned}
		X & =N_{X} ,\\
		Y & =\begin{cases}
			\alpha X+N_{Y}, & \text{Linear}\\
			f(X)+N_{Y}. & \text{Nonlinear}
		\end{cases}
	\end{aligned}
\end{equation}
where the variance $\text{Var}(N_{X})=\sigma_{N_{X}}^{2}$, $\text{Var}(N_{Y})=\sigma_{N_{Y}}^{2}$, $N_{X} \Vbar N_{Y}$, and $f$ is a nonlinear function. 

Without loss of generality, we denote the additive noise model in the anti-causal direction as:
\begin{equation} \label{eq:anti-causal anm}
	\begin{aligned}
		Y & =\hat{N}_{Y} ,\\
		X & =\begin{cases}
			\hat{\alpha} Y+ \hat{N}_{X}, & \text{Linear}\\
			g(Y)+ \hat{N}_{X}. & \text{Nonlinear}
		\end{cases}
	\end{aligned}
\end{equation}
where the noises are not independent $\hat{N}_X \nVbar \hat{N}_Y$. Then, using the least-square loss, one may identify the causal direction by comparing the difference of least-square loss between $LS_{X\to Y}$ and $LS_{Y\to X}$ on the causal direction and anti-causal discretion, receptively, and taking the smaller one as the causal direction.

However, one of the main issues of using least-square loss for causal discovery is that the score in each direction can be easily manipulated by the causal strength or the variance of noises. As shown in Figure \ref{figure 1}, we run linear regression using the least-square loss 100 times in each experiment and show the difference of least-square loss between the causal direction and the anti-causal direction with different causal strengths and variance of noises in linear Gaussian, Uniform, and Gumbel data, respectively. We can see that the least-square loss, even fails to identify the causal direction under the non-Gaussian distribution that is notably identifiable \cite{shimizu2006linear}. Thus, the least-square loss should not be a proper criterion for identifying the causal direction and a new criterion is urgently needed for methods like NOTEARS. To achieve this, we first show why the least-square loss is not a proper criterion in the following sub-sections and propose a corresponding remedy in Section \ref{sec:entropy}.

\subsection{Linear Case}

In this section, we show that using the least-square loss, the identifiability of linear ANM can be fully determined by its causal strength and the noise variance. 

To be clear, let $X \rightarrow Y$ denote the causal direction from cause $X$ to effect $Y$, and let $Y \rightarrow X$ denote the anti-causal direction. Let $\sigma_{X}^{2}$ denote the variance of variable $X$. Let $\beta_{Y|X}$ be the linear regression coefficient obtained from regressing $Y$ on $X$, and let $\sigma_{Y|X}^{2} = \text{Var}(Y - \beta _{Y|X} \cdot X)$ be the variance of the residual of regressing $Y$ on $X$. The least-square loss in causal direction $LS_{X\rightarrow Y}$ and in the anti-causal direction $LS_{X\rightarrow Y}$ can be represented by:
\begin{equation}
    \begin{aligned}
    LS_{X\rightarrow Y}&  = \mathbb{E}\left[( X-0)^{2}\right] +\mathbb{E}\left[( Y-\beta _{Y|X} \cdot X)^{2}\right], \\
    LS_{Y\rightarrow X}&  = \mathbb{E}\left[( X-\beta _{X|Y} \cdot Y)^{2}\right] +\mathbb{E}\left[( Y-0)^{2}\right]. \\
    \end{aligned}
\end{equation}
By analyzing the square-least loss in different directions, we can conclude the following theorem.

\begin{theorem} \label{theorem mseLinear}
	Let $X \to Y$ be the causal direction following the data generation mechanism $X=N_X, Y = \alpha X + N_Y$.  The least-square loss will fail to identify the correct causal direction if the causal strength $\alpha$, the noise variances $\sigma_{N_X}^{2}$ and $\sigma_{N_Y}^{2}$ satisfy the following inequality:
	\begin{equation} \label{inequality mselinear}
		\alpha ^{2} < 1 - \frac{\sigma_{N_Y}^{2}}{\sigma_{N_X}^{2}}.
	\end{equation}
	
\end{theorem}

\begin{proof}
	For the causal direction, we have:
	\begin{equation} \label{equ:linear causal}
		\begin{array}{c}
			\sigma_{X|\phi}^{2} = \sigma_{N_X}^{2}, \quad
			\beta_{Y|X} = \frac{\alpha \sigma_{X}^{2}}{\sigma_{X}^{2}} = \alpha, \quad
			\sigma_{Y|X}^{2} = \sigma_{N_Y}^{2},
		\end{array} 
	\end{equation}
	For the anti-causal direction, we have
	\begin{equation}	\label{equ:linear anti-causal}
		\begin{aligned}
			\sigma ^{2}_{Y|\phi } & =\alpha ^{2} \sigma ^{2}_{N_{X}} +\sigma ^{2}_{N_{Y}} ,\quad \beta _{X|Y} =\frac{\alpha \sigma ^{2}_{N_{X}}}{\sigma ^{2}_{Y|\phi }} =\frac{\alpha \sigma ^{2}_{N_{X}}}{\alpha ^{2} \sigma ^{2}_{N_{X}} +\sigma ^{2}_{N_{Y}}} ,\\
			\sigma ^{2}_{X|Y} & = \text{Var} (X-\beta _{X|Y} Y)=\frac{\sigma ^{2}_{N_{X}} \sigma ^{2}_{N_{Y}}}{\alpha ^{2} \sigma ^{2}_{N_{X}} +\sigma ^{2}_{N_{Y}}} .
		\end{aligned}
	\end{equation}
	Thus, using Eq. \ref{equ:linear causal}, the least-square loss in the direction of $\displaystyle X\rightarrow Y$ is given as follows:
	\begin{equation} \label{equ:mse linear causal}
		LS_{X\rightarrow Y} =\mathbb{E}\left[( X-0)^{2}\right] +\mathbb{E}\left[( Y-\beta _{Y|X} \cdot X)^{2}\right] =\sigma ^{2}_{N_{X}} +\sigma ^{2}_{N_{Y}} ,
	\end{equation}
	where  $\displaystyle \beta _{Y|X}$ is the coefficient obtained from linear regression. Intuitively, the least-square loss is to calculate the sum of the residual's variance, and therefore the loss in the causal direction is equal to the sum of noise variance. Similarly, using Eq. \ref{equ:linear anti-causal}, we obtain
	\begin{equation}\label{equ:mse linear anti-causal}
	  \begin{aligned}
		LS_{Y\rightarrow X} & =\mathbb{E}\left[( X-\beta _{X|Y} \cdot Y)^{2}\right] +\mathbb{E}\left[( Y-0)^{2}\right]  \\
        & =\frac{\sigma ^{2}_{N_{Y}} \sigma ^{2}_{N_{X}}}{\alpha ^{2} \sigma ^{2}_{N_{X}} +\sigma ^{2}_{N_{Y}}} +\alpha ^{2} \sigma ^{2}_{N_{X}} +\sigma ^{2}_{N_{Y}} .
	  \end{aligned}
	\end{equation}
	We can see that different from the least-square loss in the causal direction, the least-square loss in the anti-causal direction has much more complicated terms, which
	is closely related to the causal strength and the noise variance. Thus, if we consider the condition that the loss in the causal direction $X \rightarrow Y$, is larger than the loss in the anti-causal direction $Y \rightarrow X$, i.e., $LS_{X\rightarrow Y}  >LS_{Y\rightarrow X}$, using Eq. \ref{equ:mse linear causal} and \ref{equ:mse linear anti-causal}, we have the following inequality:
	\begin{equation*} 
		\begin{aligned}
			\sigma ^{2}_{N_{X}} +\sigma ^{2}_{N_{Y}}  & > \frac{\sigma ^{2}_{N_{Y}} \sigma ^{2}_{N_{X}}}{\alpha ^{2} \sigma ^{2}_{N_{X}} +\sigma ^{2}_{N_{Y}}} +\alpha ^{2} \sigma ^{2}_{N_{X}} +\sigma ^{2}_{N_{Y}} \\
			\alpha ^{2} & < \frac{\alpha ^{2} \sigma _{N_{X}}^{2} +\sigma _{N_{Y}}^{2} -\sigma _{N_{Y}}^{2}}{\alpha ^{2} \sigma _{N_{X}}^{2} +\sigma _{N_{Y}}^{2}}  \\
			\alpha ^{2} & < 1 - \frac{\sigma_{N_Y}^{2}}{\sigma_{N_X}^{2}}.
		\end{aligned}
	\end{equation*}
\end{proof}

A similar result has been shown in \cite{park2020identifiability} from the view of conditional variances, of which the main contribution is to prove the identifiability of additive noise model with unknown heterogeneous error variances.

Theorem \ref{theorem mseLinear} indicates that, under the linear additive noise model assumption, the least-square loss fails to identify the correct causal direction unless the inequality \ref{inequality mselinear} does not hold. It is also interesting to see that if all noises follow the standard Gaussian distribution, we have $\sigma_{N_X}^{2} = \sigma_{N_Y}^{2}$, then the causal direction can still be correctly identified because the inequality $\alpha^2<0$ must not hold. However, it is unrealistic that all noises follow the same standard Gaussian distribution in the real world. And a simple standardization that is usually used for preprocessing will invalidate the assumption. 

Moreover, Theorem \ref{theorem mseLinear} also explains the phenomenon in Figure \ref{figure 1}. For example, based on Eq. \ref{equ:mse linear causal} and Eq. \ref{equ:mse linear anti-causal}, given a fix $\sigma_{N_X}^2=4$ and $\sigma_{N_Y}^2=1$, then we have $LS_{X\to Y}=5$ and $LS_{Y\to X}=\frac{4}{4\alpha^2+1}+4\alpha^2+1$. As a result, the $LS_{Y\to X}$ will be even smaller than $LS_{X\to Y}$ if we have $|\alpha|<\frac{\sqrt{3}}{2}\approx0.866$, and reaches its minimum $LS_{Y\to X}=4$ when $|\alpha|=0.5$. As a result, the causal direction will be misdirected as the anti-causal direction has a smaller least-square loss.

Therefore, we conclude that the least-square loss is not a suitable score for learning linear causal structures.

\subsection{Nonlinear Case}	\label{sec:nonlinear}

In this section, we show that using the least-square loss, the identifiability of ANM is closely related to the measured mutual information between the cause and regression residual in the nonlinear case. 

Specifically, we study SCM between $X \to Y$ with the nonlinear additive form $Y=f(X)+N_{Y}$. To show the limitation of least-square loss in the nonlinear case, we aim to show that using the least-square loss can be viewed as minimizing the mutual information between the noise and cause variables but with bias such that the independence measure will be not reliable, leading to identifying the wrong direction, which is the main theoretical result of this section, i.e., Theorem \ref{theorem:mse nonlinear}. The proof of Theorem \ref{theorem:mse nonlinear} depends on Lemma \ref{lemma 1}, \ref{lemma 2}, and \ref{lemma 3}. To clearly show how these theories connect, we provide the proof sketch in Figure \ref{fig:proof sketch}. 

\begin{figure}[!t] 
	\centering
	\includegraphics[width=0.45\textwidth]{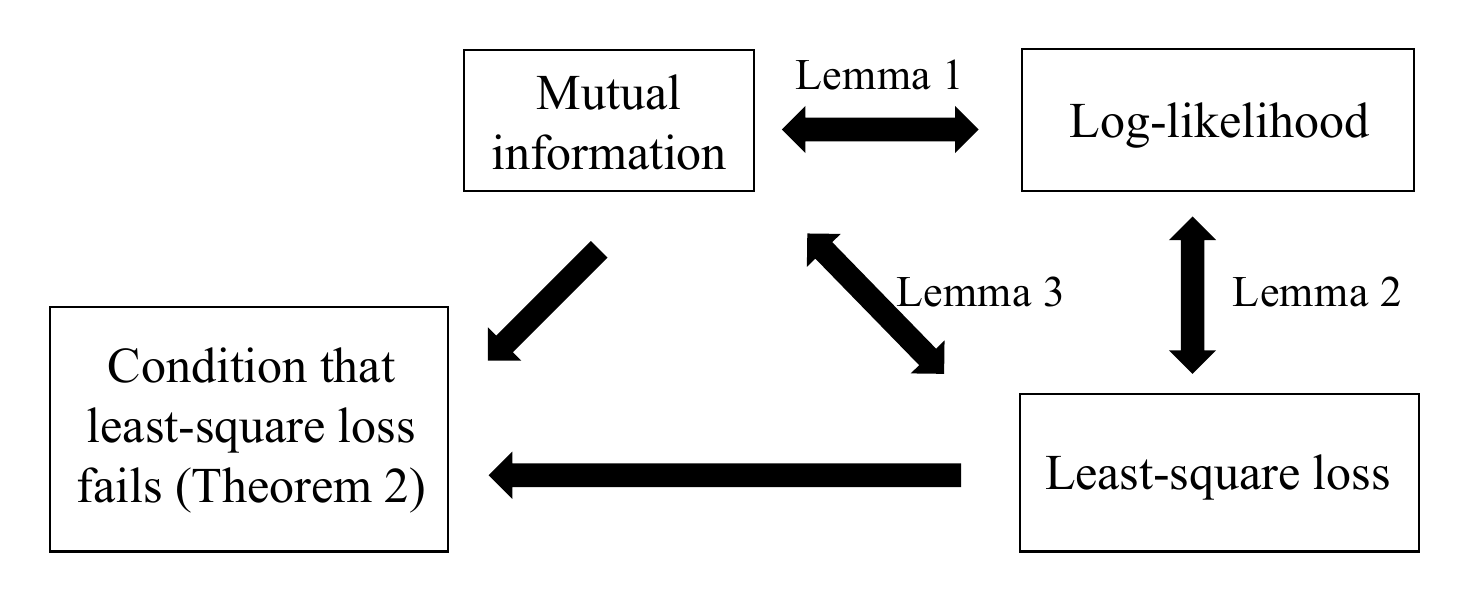} 	
	\caption{Proof sketch in Section \ref{sec:nonlinear}. We build the connection between mutual information, log-likelihood, and least-square loss. Based on such a connection, we develop the condition that least-square loss fails to identify the causal direction (Theorem \ref{theorem:mse nonlinear}).} \label{fig:proof sketch}
\end{figure}

\textbf{Proof Sketch}: First, we show that maximizing log-likelihood is equivalent to minimizing the mutual information between the cause and the noise variables (Lemma \ref{lemma 1}). Second, we further show that maximizing the likelihood under the standard Gaussian noise assumption is equivalent to minimizing the least-square loss (Lemma \ref{lemma 2}). Third, by combing Lemma \ref{lemma 1} and Lemma \ref{lemma 2}, we conclude that using least-square loss, the mutual information will be biased (Lemma \ref{lemma 3}). Finally, by considering the bias in mutual information, we provide the condition that least-square loss fails to identify the causal direction in Theorem \ref{theorem:mse nonlinear}.

\begin{lemma} \label{lemma 1} 
	Given the samples $\left\{x^{( i)} ,y^{( i)}\right\}^{m}_{i=1}$ and causal model $Y=f(X;\theta)+N_{Y} $ with any parameters $\displaystyle \theta $ and as $m \rightarrow \infty$, the average log-likelihood $l_{X\rightarrow Y} (\theta )$ and the mutual information between $X$ and $N_Y$ are related in the following way:
	\begin{equation} \label{main eq:mutual information vs likelihood}
		I(X,N_{Y} ;\theta )
		= \mathbb{E}_{x,y\sim p(X,Y)} \left[ p(X=x,Y=y) \right] - \lim_{m \to \infty} l_{X\rightarrow Y} (\theta ),
	\end{equation}
	where
	\begin{equation} \label{main eq:general_likelihood}
	  \begin{aligned}
		&l_{X\rightarrow Y} (\theta ) \\
		=& 
		\frac{1}{m} \sum ^{m}_{i=1}\log p(X=x^{( i)} )
		+ \frac{1}{m}\sum ^{m}_{i=1}\log p(N_{Y}=y^{(i)}-f(x^{(i)}) ;\theta ) .
	  \end{aligned}
	\end{equation}
\end{lemma}

The proof is in Appendix \ref{app_sec: proof of lemma 1}. Lemma \ref{lemma 1} indicates that in the pairwise ANM, maximizing the log-likelihood (Eq. \ref{main eq:general_likelihood}) is equivalent to minimizing the mutual information (Eq. \ref{main eq:mutual information vs likelihood}). 
Next, we further build a connection between the log-likelihood and the least-square loss in the following lemma.

\begin{lemma} \label{lemma 2}
	For the additive noise model $Y=f(X)+N_Y$, as $m \rightarrow \infty$, maximizing the average log-likelihood $l_{X\to Y}(\theta)$ with the standard Gaussian noise assumption,
	\begin{equation} \label{main eq:lik_gaussian}
		\begin{aligned}
			l_{X\rightarrow Y} (\theta ) = & \frac{1}{m}
			\sum ^{m}_{i=1}\log\left(\frac{1}{\sqrt{2\pi }}\exp\left( -\frac{(x^{( i)} -0)^{2}}{2}\right)\right) \\
			& + \frac{1}{m} \sum ^{m}_{i=1}\log\left(\frac{1}{\sqrt{2\pi }}\exp\left( -\frac{(y^{( i)} -f (x^{( i)} ;\theta ))^{2}}{2}\right)\right) ,
		\end{aligned}
	\end{equation}
	is equivalent to minimizing the least-square loss $LS_{X \to Y}$,
	\begin{equation} \label{main eq:lemma_lsloss}
		LS_{X\rightarrow Y} =\mathbb{E}\left[ (X-0)^{2}\right] +\mathbb{E}\left[ (Y-f (X))^{2}\right] .
	\end{equation}
\end{lemma}

The proof is in Appendix \ref{app_sec: proof of lemma 2}. Lemma \ref{lemma 2} indicates that under the standard Gaussian noise assumption, minimizing the log-likelihood (Eq. \ref{main eq:lik_gaussian}) is equivalent to minimizing the least-square loss (Eq. \ref{main eq:lemma_lsloss}). It implies that choosing the least-square loss as the objective function will introduce the standard Gaussian distribution in log-likelihood, causing a distribution mismatch when the underlying noise distribution is not the standard Gaussian. That is, the least-square loss will introduce bias to the mutual information according to the connection between log-likelihood and mutual information in Lemma \ref{lemma 1}. To be precise, we have the following lemma.

\begin{lemma} \label{lemma 3}
	For the additive noise model, as $m \rightarrow \infty$, minimizing the least-square loss is equivalent to minimizing the mutual information under the standard Gaussian noise assumption with the following form:
	\begin{equation} \label{main eq:gaussian mutual information}
	 \begin{aligned}
			   I_q(X, &  N_{Y} ;\theta ) 
			  =  \mathbb{E}_{x,y\sim p(X,Y)}\left[\log p(X=x,Y=y) 
			  \right. \\
			  &\left. -\log q(X=x) -\log q(N_Y=y - f(x) ;\theta ) \right] ,
	  \end{aligned}
	\end{equation}
	where $\displaystyle q$ is the density function of standard Gaussian distribution.
\end{lemma}

The proof is in Appendix \ref{app_sec: proof of lemma 3}.
Lemma \ref{lemma 3} shows that the least-square loss is equivalence to the biased mutual information given in Eq. \ref{main eq:gaussian mutual information} which explains exactly why the least-square loss will fail to identify the correct causal diction. The reason is that based on the independence property of ANM, i.e., $X \Vbar N_Y$ but $Y \nVbar \hat{N}_X$, using the standard mutual information is able to capture such independence property and to identify the causal direction by testing whether the noise and cause are independent in the causal direction but not independent in the anti-causal direction, i.e., $I(Y, \hat{N}_{X})>I(X, N_{Y})=0$. However, if the least-square loss is used and the underlying noise distribution is not standard Gaussian, the mutual information is biased as stated in Lemma \ref{lemma 3} such that the independence property $I_q(Y, \hat{N}_{X})>I_q(X, N_{Y})$ does not necessarily hold. That is, in the following theorem, we show that when the condition $I_q(Y, \hat{N}_{X})>I_q(X, N_{Y})$  occurs, the least-square loss will fail to identify the causal direction.

\begin{theorem} Let $X \to Y$ be the causal direction following the data generation mechanism $Y = f(X) + N_Y$ and we assume $m \rightarrow \infty$.  Using least-square loss the causal direction is non-identifiable if the following inequality holds: \label{theorem:mse nonlinear}
	\begin{equation} \label{main inequality:mse nonlinear}
		\begin{aligned}
			-\int p(X)\log q(X)dX -\int p(N_{Y})\log q(N_{Y})dN_{Y}  \\
			> -\int p(Y)\log q(Y)dY -\int p(\hat{N}_{X} )\log q(\hat{N}_{X} )d\hat{N}_{X},
		\end{aligned}
	\end{equation}
	where q is the density function of standard Gaussian distribution.
\end{theorem}

The proof is in Appendix \ref{app_sec: proof of theorem 2}. Theorem \ref{theorem:mse nonlinear} gives the certain condition of non-identifiability of the least-square loss. The key reason why the least-square can not correctly identify causal direction is the distribution mismatch between the underlying distribution $p$ and standard Gaussian distribution $q$. As shown in inequality \ref{main inequality:mse nonlinear}, when the divergence between the standard Gaussian distribution $q$ and the underlying distribution $p(X)$ or $p(N_Y)$ is large, the left-hand side will be larger than the right-hand side, making the causal direction misdirected using the least-square loss.

For example, let $Y=\text{sigmoid}(X) +N_{Y}$ where $X \sim \text{Uniform}\left(-C, C\right)$ and $ N_{Y} \sim \text{Uniform}\left( -1,1\right)$. In this case, we can simply increase $C$ such that $-\int p(X)\log q(X) dX$ will tend to infinity while $Y$ and $\hat{N}_X$ are bounded due to the sigmoid function, such that the right-hand side of Eq. \ref{main inequality:mse nonlinear} is also bounded, and hence the inequality must hold. In this case, the least-square loss will identify incorrectly the causal direction.

Therefore, it is necessary to use the correct distribution setting $q=p$ to obtain the correct result, which inspires a way that uses the entropy-based loss instead.

\section{Structure Learning Using Entropy-based Loss} \label{sec:entropy}

As discussed in Section \ref{sec:Limitations}, using least-square loss, which is equivalent to assuming the distribution of all noise terms is standard Gaussian distribution, will result in incorrect causal identification. The reason for these errors is the mismatch between the underlying noise distribution and the assumed standard Gaussian distribution.

Instead of using the square operator in least-square loss, we can replace it with the entropy as follows:
\begin{equation} \label{loss entropy}
	\min \sum_{i=1}^d H(N_i) + \lambda \left\|W\right\|_1 \quad \text{subject to} \quad tr\left(e^{W \circ W}\right)-d = 0,
\end{equation}
where $N_i$ denotes the independent noise and $H(N_i)$ denotes the entropy of noise $N_i$.

\begin{figure}[!h] 
	\centering
	\subfloat[Residuals in the regression of  casual direction, $N_X \Vbar N_Y$]{ 
		\includegraphics[width=0.22\textwidth]{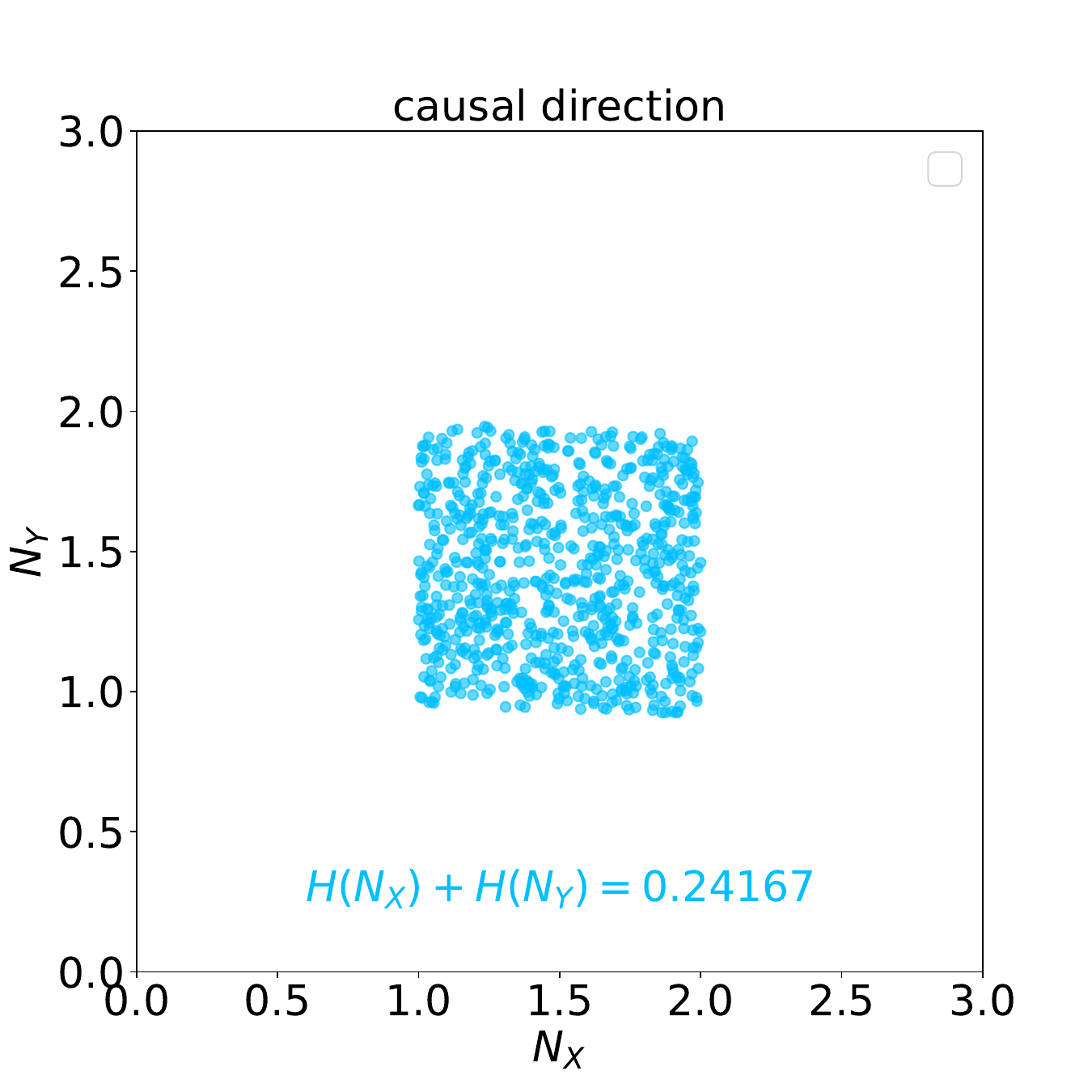} 
		\label{figure intuition a}
	}
	\hfil
	\subfloat[Residuals in the regression of anti-casual direction, $\hat{N}_X \nVbar \hat{N}_Y$]{
		\includegraphics[width=0.22\textwidth]{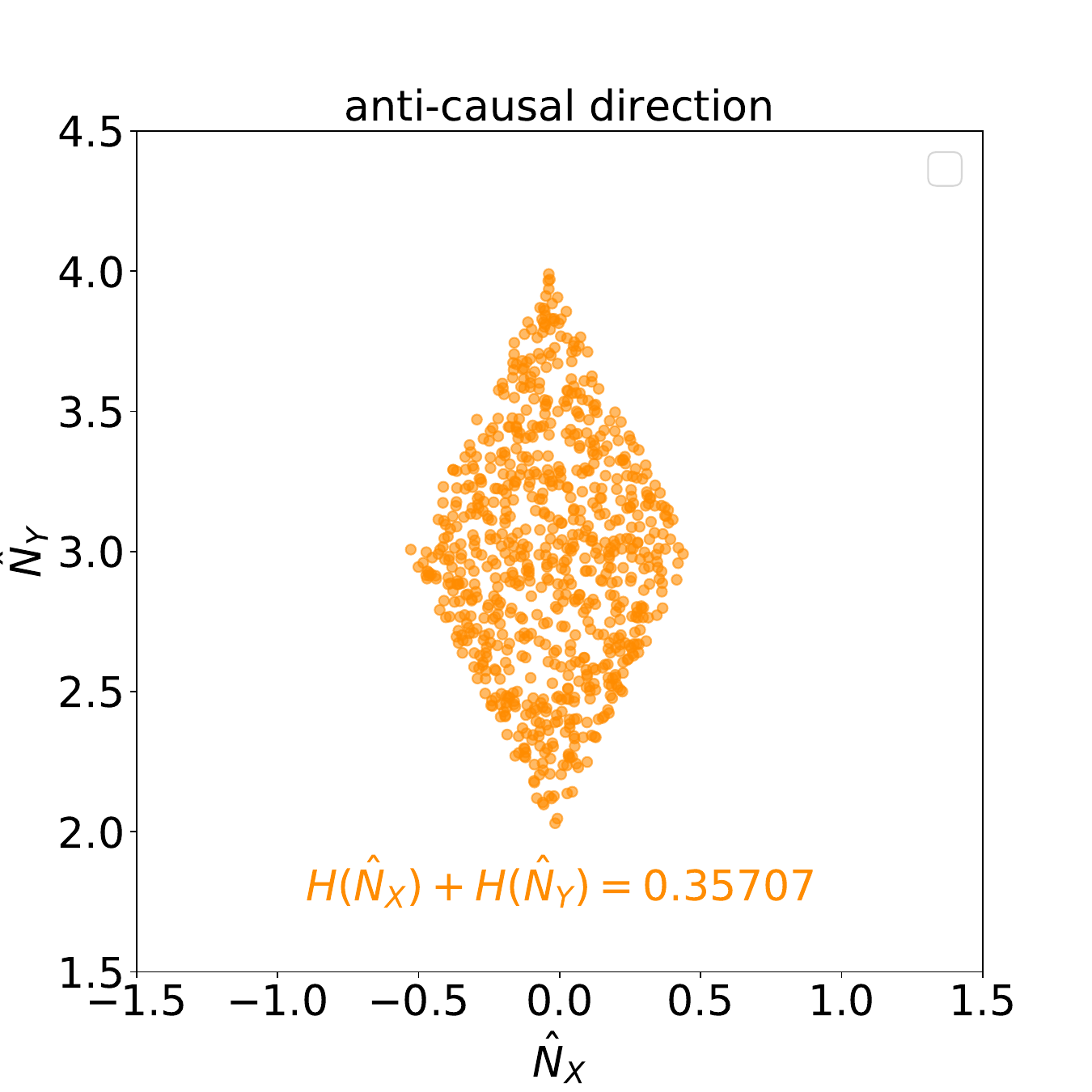}
		\label{figure intuition b}
	}
	\caption{Visualization of residuals in the regression of causal and anti-causal directions, respectively.}
	\label{figure intuition} 
\end{figure}

The intuition that we use entropy-based loss is shown in Figure \ref{figure intuition}. Intuitively, the sum of entropy can be rewritten as follows $H(N_X)+H(N_Y)=H(N_X,N_Y)+I(N_X,N_Y)$. As a result, since all residuals in the causal direction are independent of each other, we should have the lower sum of entropy under which $I(N_X,N_Y)=0$. While the residuals in the anti-causal direction depend on each other, we should have the larger sum of entropy as the mutual information $I(\hat{N}_X,\hat{N}_Y)>0$.

To explain why we use entropy-based loss theoretically, we first show that the entropy-based loss is consistent with the log-likelihood score under the additive noise model (see Theorem \ref{theorem entropy vs likelihood}). Based on Theorem \ref{theorem entropy vs likelihood}, by utilizing identifiability results in terms of likelihood (see Lemma \ref{ANM}), it is easy to show that the entropy-based loss is also identifiable (see Corollary \ref{corollary:ent asymmetry}).

\begin{theorem} \label{theorem entropy vs likelihood}
	In the additive noise model, the entropy-based score has a consistency with the log-likelihood score when the sample sizes $m \rightarrow \infty $, i.e.,
	\begin{equation}
		\lim_{m \to \infty}\frac{1}{m}\sum ^{m}_{j=1}\sum ^{d}_{i=1}\log p\left( x^{(j)}_{i} |x^{( j)}_{pa( i)}\right) =-\sum ^{d}_{i=1} H(N_{i} ).
	\end{equation}
\end{theorem}

The proof is in Appendix \ref{app_sec: proof of theorem 3}.
Theorem \ref{theorem entropy vs likelihood} indicates that, under the additive noise model, the entropy-based loss is consistent with the log-likelihood score, while Peters et al. \cite{JMLR:v15:peters14a} had proven that for the additive noise model, the log-likelihood is able to distinguish the causal direction under a mild assumption, which is illustrated in the following lemma.

\begin{lemma}[Theorem 1 in \cite{hoyer2008nonlinear}] \label{ANM}
	For the additive noise model $Y = f(X)+N_Y$, there is a forward model of the form $p(X,Y)=p(Y-f(X))p(X)$. If there is a backward model of the same form $p(X,Y)=p(X-g(Y))p(Y)$, then for all $X$, $Y$ with the three-time differentiable $f$ and $v^{\prime\prime}(Y-f(X))f^{\prime}(X) \neq 0$, the following equality holds:
	\begin{equation} \label{eq:anm condition}
		\xi^{\prime\prime\prime} = \xi^{\prime\prime} \left( - \frac{v^{\prime\prime\prime} f^{\prime}}{v^{\prime\prime}} + \frac{f^{\prime\prime}}{f^{\prime}} \right) - 2v^{\prime\prime}f^{\prime\prime}f^{\prime} + v^{\prime}f^{\prime\prime\prime} + \frac{v^{\prime}v^{\prime\prime\prime}f^{\prime\prime}f^{\prime}}{v^{\prime\prime}} - \frac{v^{\prime}(f^{\prime\prime})^2}{f^{\prime}},
	\end{equation}
	where $v \coloneq \log p(N_Y)$, $\xi\coloneq \log p(X)$.
\end{lemma}

Lemma \ref{ANM} describes a necessary condition of the identifiability of ANM, i.e., if there exists a backward ANM with the independence property between cause and noise, the condition Eq. \ref{eq:anm condition} will hold. In other words, if the condition Eq. \ref{eq:anm condition} does not hold, there must not exist a backward ANM such that the cause and noise are independent and we can identify the causal direction using the log-likelihood score. Fortunately, such a condition only holds in restricted cases, e.g., the bivariate linear Gaussian noises model \cite{peters2014causal} and thus ANM can be identified in most cases.

Given backward ANM does not exist, i.e., Eq. \ref{eq:anm condition} does not hold, we generalize the result in Lemma \ref{ANM} to the following corollary by incorporating the result of Theorem \ref{theorem entropy vs likelihood}, showing that the entropy-based loss is able to identify the causal direction.

\begin{corollary} \label{corollary:ent asymmetry}
	For each pair of additive noise model $Y=f(X)+N_Y$, if $v^{\prime\prime}(Y-f(X))f^{\prime}(X) \neq 0$ and the condition in Eq. \ref{eq:anm condition} does not hold, then using the entropy-based loss each pair of additive noise model is identifiable and the following inequality holds: 
	\begin{equation} \label{main: inequality: entropy}
		H\left(X\right) + H\left(N_{Y}\right) < H\left(Y\right) + H\left(\hat{N}_{X}\right).
	\end{equation}
\end{corollary}

The proof is in Appendix \ref{app_sec: proof of corollary 1}.
Corollary \ref{corollary:ent asymmetry} shows the identifiability of entropy-based loss under ANM assumption. It indicates that the entropy of residuals is lower in the causal direction than in the anti-causal direction, making the causal direction can be identified by entropy. The example in Figure \ref{figure intuition} also shows such an asymmetry of the entropy of residuals between the causal direction and the anti-causal direction. The regression residuals in the causal direction (blue points in Figure \ref{figure intuition a}) are independent, while the regression residuals in the anti-causal direction (yellow points in Figure \ref{figure intuition b}) are not independent. This independence property is characterized by the entropy of residuals, which is consistent with the inequality \ref{main: inequality: entropy} in Corollary \ref{corollary:ent asymmetry}.

Moreover, for the task of causal discovery, corollary \ref{corollary:ent asymmetry} illustrates the validity of the entropy-based loss under the additive noise model, i.e., Eq. \ref{loss entropy}, which is based on the consistency with the likelihood score, indicating that by comparing the entropy of residuals between causal and anti-causal directions, we can successfully identify the causal direction.

\begin{figure*}[!t] 
	\centering
	\subfloat[Accuracy varies with various causal strength]{
		\includegraphics[height=0.17\textwidth]{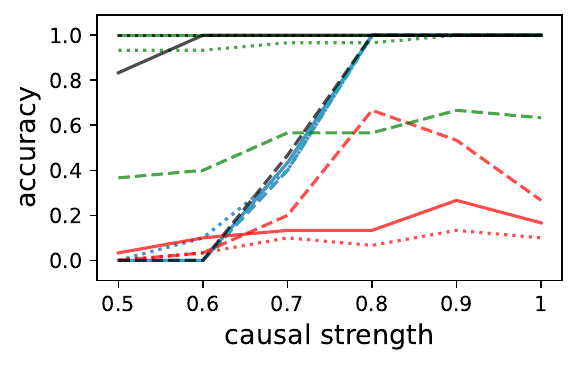} 
		\label{figure pair causal strength}
	}
	\hfil
	\subfloat[Accuracy varies with various variance of $N_{X}$]{
		\includegraphics[height=0.17\textwidth]{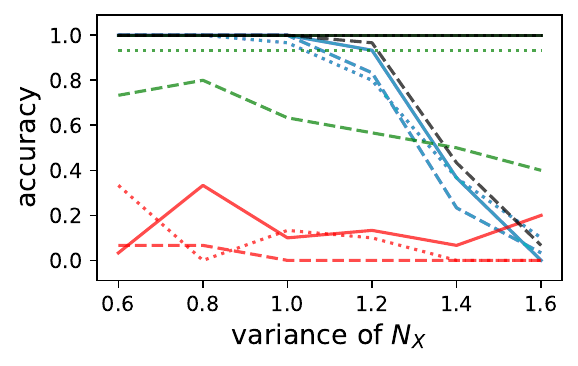}
		\label{figure pair nx}
	}
	\hfil
	\subfloat[Accuracy varies with various variance of $N_{Y}$]{
		\includegraphics[height=0.17\textwidth]{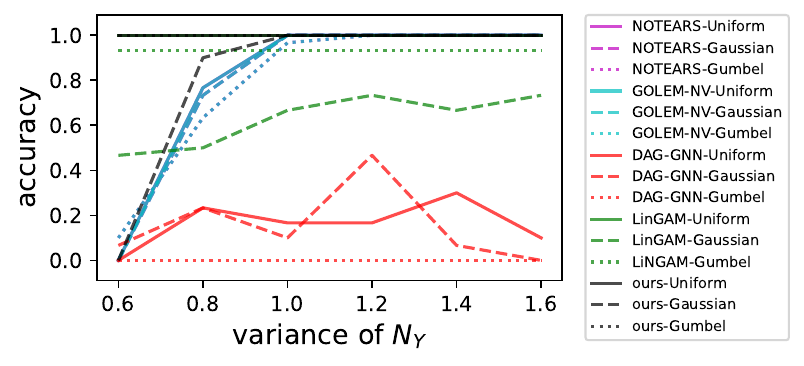}
		\label{figure pair ny}
	}
	
	\caption{Causal direction accuracy in pairwise data: the default parameters are as follows: $\alpha=0.5$, $\sigma_{N_{X}}=2$, and $\sigma_{N_{Y}} = 1$. And the controlled parameters will vary in the range as follows: $\alpha \in [0.5,1]$, $\sigma_{N_{X}} \in [0.6,1.6]$, $\sigma_{N_{Y}}\in [0.6,1.6]$. The blue lines are caused by the total overlap between NOTEARS and GOLEM-NV.
	} \label{figure pair} 
\end{figure*}

\section{Experiments} \label{sec:experiment}

To investigate the effectiveness of the entropy-based loss, we compare it with baseline algorithms on both synthetic data and real-world data. The synthetic data are generated from linear pairwise SCM, linear non-Gaussian SCM, and nonlinear additive noise SCM, respectively. The following algorithms are taken as the baseline: NOTEARS \cite{zheng2018dags}, LiNGAM \cite{shimizu2006linear}, GOLEM-NV \cite{ng2020role}, DAG-GNN \cite{yu2019dag}, NOTEARS-MLP \cite{zheng2020learning}, GraN-DAG \cite{lachapelle2019gradient}. We reuse the parameter settings for those baseline algorithms in their original papers and codes.

Our implementation, denoted by \textbf{ours} (linear) and \textbf{ours-MLP} (nonlinear), is based on the codes of NOTEARS \cite{zheng2018dags} and NOTEARS-MLP \cite{zheng2020learning}, respectively. Following the previous works, the numerical optimize algorithms L-BFGS \cite{byrd1995limited} is used for optimization, and we will prune the edges after training with a small threshold $\omega = 0.3$ to rule out cycle-inducing edges.

In the linear pairwise SCM, the data are generated according to the following SCM: $X = N_X, Y = \alpha X + N_Y$ with causal strength $\alpha  \in [ 0.5,1.0 ]$ and the noise $N_X, N_Y \sim \sigma_{N}*\text{Gaussian}(0, 1)$ or $N_X, N_Y \sim \sigma_{N}*\text{Uniform}(-\sqrt{3},\sqrt{3})$ or $N_X, N_Y \sim \sigma_{N}*\text{Gumbel}(0,\frac{\sqrt{6}}{\pi})$ with standard deviation $\sigma_{N} \in [0.6, 1.6]$. 

In linear non-Gaussian case, the data are generated according to the following linear SCM: $ X_i = \Sigma_{j \in pa(i)} \beta_{ij} X_j + N_j$ with random causal strength $\beta_{ij} \sim \text{Uniform}(-0.4, -0.8) \cup \text{Uniform}(0.4, 0.8) $ and the noise $N_j \sim \sigma_{N_j}*\text{Uniform}(-\sqrt{3},\sqrt{3})$ or $N_j \sim \sigma_{N_j}*\text{Gumbel}(0,\frac{\sqrt{6}}{\pi})$ with random standard deviation $\sigma_{N_j} \sim \text{Uniform}(0.5, 1.0)$. 

In the nonlinear case, the data are generated according to the following nonlinear SCM: $X_i = \text{tanh}\left(\Sigma_{j \in pa(i)} \beta_{1,ij} X_j \right) + \text{cos}\left(\Sigma_{j \in pa(i)}\beta_{2,ij} X_j \right) + \text{sin}\left(\Sigma_{j \in pa(i) }\beta_{3,ij} X_j \right) + N_j$ with random causal strength $\beta_{k,ij} \sim \text{Uniform}(0.5, 2.0)$ for $k=1,2,3$ and the noise $N_j \sim \sigma_{N_j}*\text{Uniform}(-\sqrt{3},\sqrt{3})$ with random standard deviation $\sigma_{N_j} \sim \text{Uniform}(0.5, 1.0)$.

Following previous works, Structural Hamming Distance (SHD), False Discovery Rate (FDR), and True Positive Rate (TPR) are recorded as the evaluation metrics for all algorithms.

\subsection{Pairwise Synthetic Data}

In this section, we extend the experiments in Figure \ref{figure 1}, and design a series of controlled experiments with respect to the causal strength, the variance of cause, and the variance of effect. The number of variables is fixed at 2. At each experiment, we will control one of the parameters while fixing others. The range of the above controlled parameters are as follows: $\alpha = \{0.5, 0.6, 0.7, 0.8, 0.9, 1.0\}$, $\sigma_{X}=\{0.6, 0.8, 1.0, 1.2, 1.4, 1.6\}$ and $\sigma_{Y}=\{0.6, 0.8, 1.0, 1.2, 1.4, 1.6\}$ where the default parameters are $\beta=0.5, \sigma_{X}=2, \sigma_{Y}=1$. All experiments will run 30 times.

As shown in Figure \ref{figure pair}, the least-square loss based methods are sensitive to causal strength and variance of noise, which verifies Theorem \ref{theorem mseLinear} indicating that they fail when $\alpha ^{2} < 1 - \frac{\sigma_{N_Y}^{2}}{\sigma_{N_X}^{2}}$. NOTEARS and GOLEM-NV, based on least-square loss, only correctly orient causal direction when $\alpha ^{2} > 1 - \frac{\sigma_{N_Y}^{2}}{\sigma_{N_X}^{2}}$ regardless of noise's distribution. DAG-GNN returns an empty adjacency matrix in many cases, and the reason might be that the auto-encoder framework is overly complex and underfitting such a simple SCM. Based on the non-Gaussian property, our method and LiNGAM perform well in the Uniform and Gumbel cases and perform poorly in the Gaussian cases, which also verifies the Corollary \ref{corollary:ent asymmetry}. The experiments show that the entropy-based loss can correctly orient the causal direction with the help of non-Gaussian property, while the least-square-based loss orients casual direction by the inequality $\alpha ^{2} > 1 - \frac{\sigma_{N_Y}^{2}}{\sigma_{N_X}^{2}}$, which does not always hold in reality.

\subsection{Synthetic Structures}
In this section, we design a series of controlled experiments with respect to the sample sizes and the number of variables on the synthetic random causal structures. At each experiment, we will control one of the parameters while fixing others. The range of the above controlled parameters are as follows: the number of samples=$\{200, 400, \textbf{600}, 800, 1000\}$ and the number of variables$=\{5, 10, \textbf{15}, 20, 25, 50, 100\}$ with in-degree=2 where the default setting is marked as bold. All experiments will run at least 10 times.

\subsubsection{Linear Case}

\begin{figure*}[!ht] 
	\centering
	\includegraphics[scale=0.3]{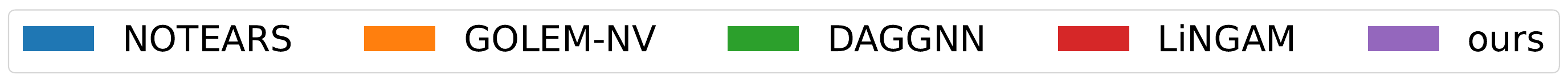}

    \subfloat[Uniform $\quad$ noises $\quad$ with $\quad$ variables=15]{		\label{figure 2a}
		\begin{minipage}[b]{0.31\textwidth}
			\includegraphics[width=1.\textwidth]{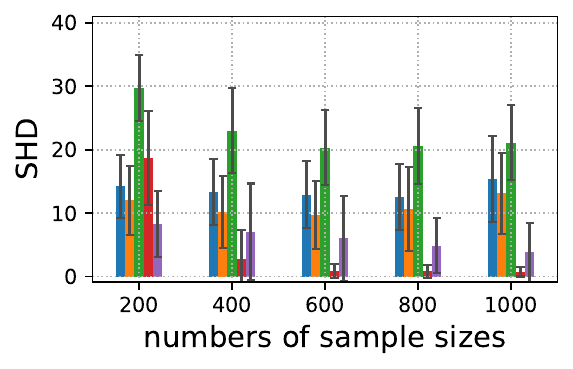} \\
			\includegraphics[width=1.\textwidth]{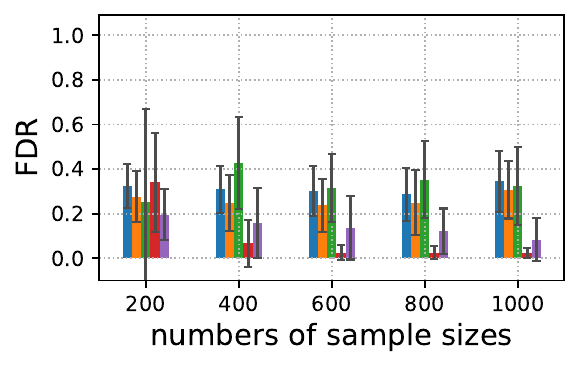} \\
			\includegraphics[width=1.\textwidth]{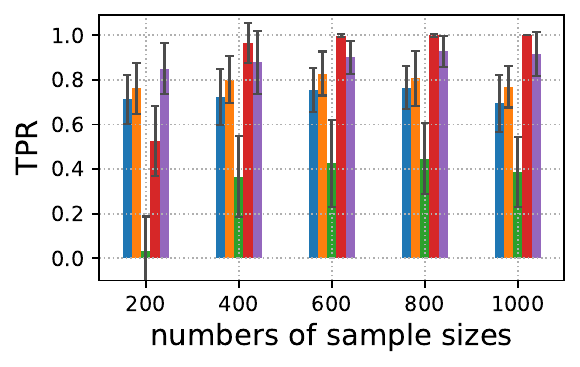}
			\hspace{0.2cm}
		\end{minipage}
	}
	\subfloat[Uniform $\quad$ noises $\quad$ with $\quad$ samples=600]{		\label{figure 2b}
		\begin{minipage}[b]{0.31\textwidth}
			\includegraphics[width=1.\textwidth]{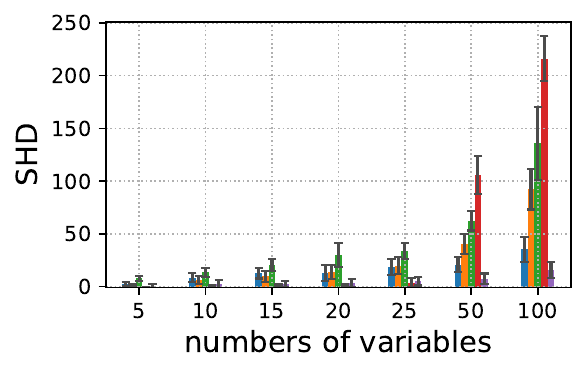} \\
			\includegraphics[width=1.\textwidth]{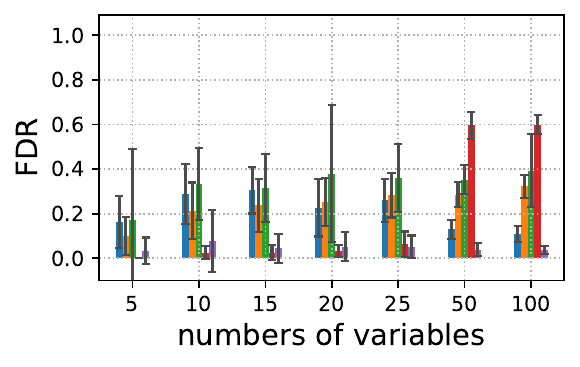} \\
			\includegraphics[width=1.\textwidth]{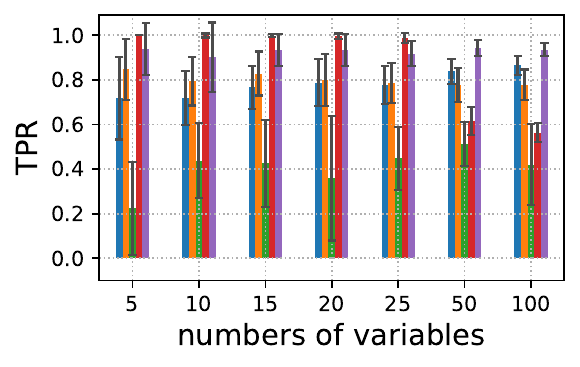} 
			\hspace{0.2cm}
		\end{minipage}
	}
	\subfloat[Gumbel $\quad$ noises $\quad$ with $\quad$ samples=600]{  		\label{figure 2c}
		\begin{minipage}[b]{0.31\textwidth}
			\includegraphics[width=1.\textwidth]{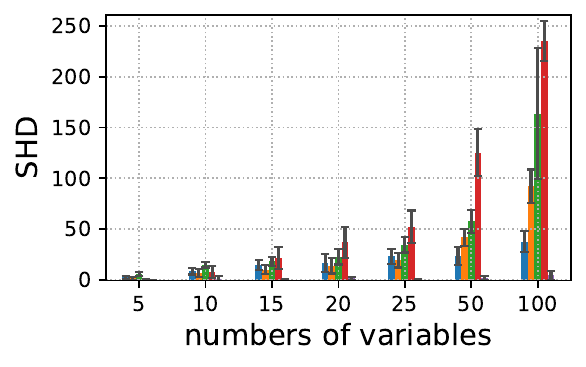} \\
			\includegraphics[width=1.\textwidth]{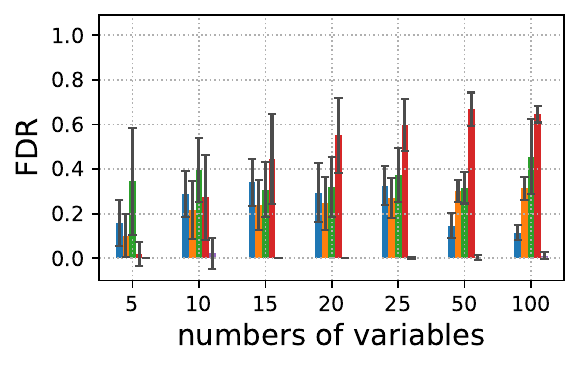} \\
			\includegraphics[width=1.\textwidth] {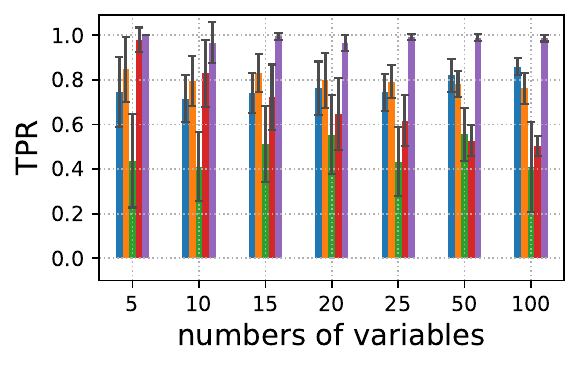} 
			\hspace{0.2cm}
		\end{minipage}
	}
	\caption{Structure recovery in linear data in terms of SHD, FDR, and TPR to the true graph: in Figure \ref{figure 2a}, the noises are Uniform, the number of variables and edges are 15 and 30; in Figure \ref{figure 2b}, the noises are Uniform and the sample size is 600; in Figure \ref{figure 2c}, the noises are Gumbel and the sample size is 600.
	} \label{figure 2} 
\end{figure*}

\begin{figure*}[!ht]  
	\centering
	\includegraphics[scale=0.3]{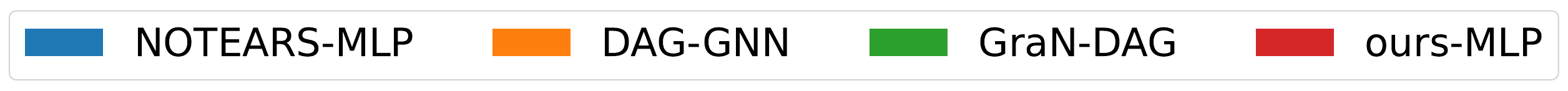}

    \subfloat[Uniform $\quad$ noises $\quad$ with $\quad$ variables=15]{		\label{figure 3a}
		\begin{minipage}[b]{0.31\textwidth}
			\includegraphics[width=1.\textwidth]{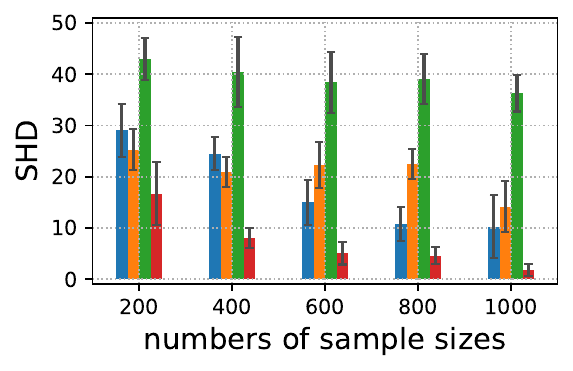} \\
			\includegraphics[width=1.\textwidth]{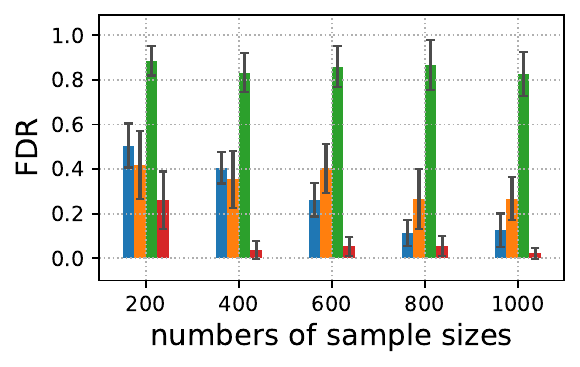} \\
			\includegraphics[width=1.\textwidth]{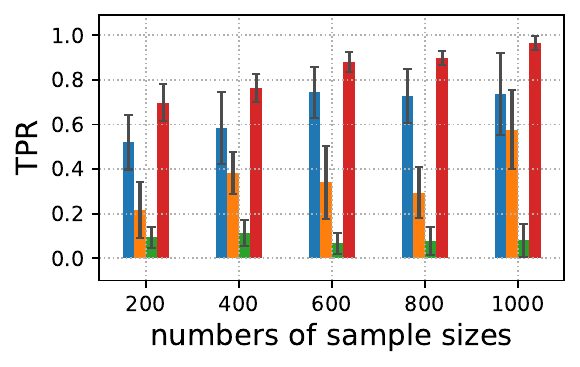} 
			\hspace{0.2cm}
		\end{minipage}
	}
	\subfloat[Uniform $\quad$ noises $\quad$ with $\quad$ samples=600]{ \label{figure 3b}
		\begin{minipage}[b]{0.31\textwidth}
			\includegraphics[width=1.\textwidth]{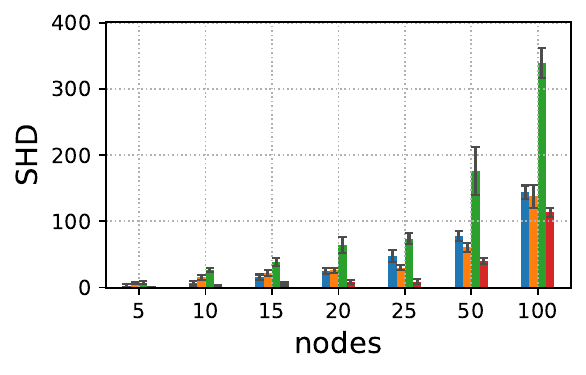} \\
			\includegraphics[width=1.\textwidth]{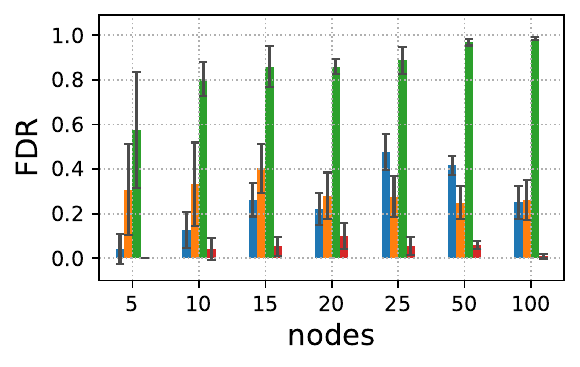} \\
			\includegraphics[width=1.\textwidth]{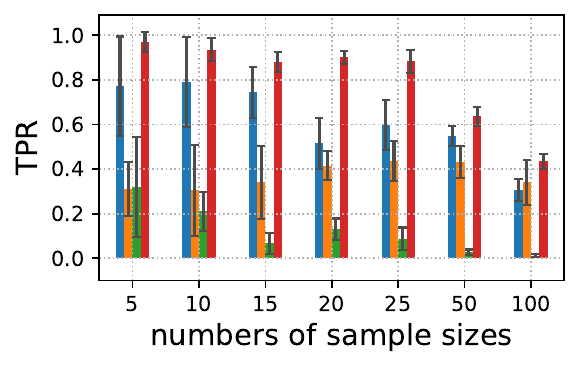} 
			\hspace{0.2cm}
		\end{minipage}
	}
	\subfloat[Uniform $\quad$ noises $\quad$ with $\quad$ variables=15]{ \label{figure 3c}
		\begin{minipage}[b]{0.31\textwidth}
			\includegraphics[width=1.\textwidth]{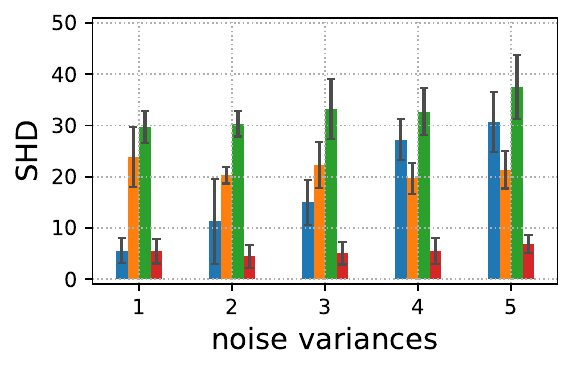} \\
			\includegraphics[width=1.\textwidth]{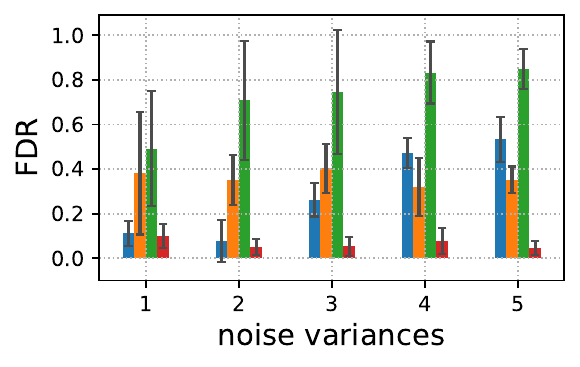} \\
			\includegraphics[width=1.\textwidth]{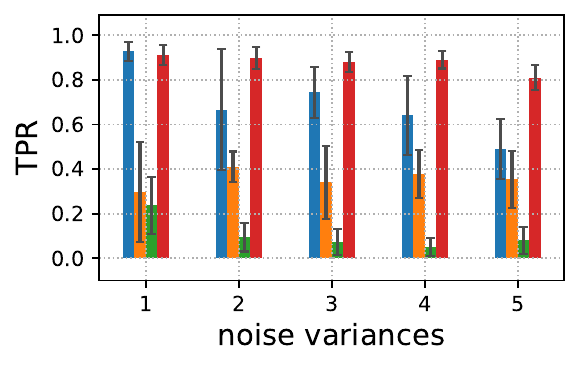}
			\hspace{0.2cm}
		\end{minipage}
	}
	
	\caption{Structure recovery in nonlinear data in terms of SHD, FDR, and TPR to the true graph: in figure \ref{figure 3a}, the noises are uniform with variance=3.0 and the number of variables and edges are 15 and 30; in Figure \ref{figure 3b}, the noises are Uniform with variance=3.0 and the number of sample sizes is 600; in Figure \ref{figure 3c}, the noises are Uniform, the number of sample sizes is 600 and the number of variables and edges are 15 and 30.
	} \label{figure 3} 
\end{figure*}
In this part, we perform linear experiments on the synthetic linear non-Gaussian data. Here, we test our method on Uniform and the Gumbel distribution, respectively. Each algorithm runs 30 times for all experiments. 

As shown in Figure \ref{figure 2}, our method outperforms all baseline methods on the non-Gaussian distributions experiments including Uniform and Gumbel distributions. The lower SHD means the causal structure recovered by our method is closer to the ground truth. The lower FDR and higher TPR mean we always recover more real edges without less misorientation. Moreover, by using the entropy-based method, we have the least deviation error, indicating the stability of our method. It is also interesting to see the significant difference between ours and NOTEARS, which verifies that the algorithm could learn the anti-causal direction using least-square loss. 
In detail, in Figure \ref{figure 2a}, as the sample size increases, all methods have better performances, and we can see that when the sample size is small, we still outperform the other gradient-based methods, which shows the robustness of our method. In Figure \ref{figure 2b} and \ref{figure 2c}, in general, the performance decreases as the number of variables increases, but compared with the baseline methods, our method decreases more slowly. Comparing Figure \ref{figure 2a} with Figure \ref{figure 2b} and \ref{figure 2c}, we find that LiNGAM performs better in low-dimension cases, but its performance decreases rapidly with the increasing of the number of variables, which indicates that gradient-based methods are more robust to deal with high dimension cases. In addition, we can see that under different distributions, our method has similar performance within experimental errors, which verifies the results in Theorem \ref{theorem mseLinear}.

\subsubsection{Nonlinear Case}

In this part, we test our entropy-based loss in nonlinear data. To further verify Theorem \ref{theorem entropy vs likelihood} and Corollary \ref{corollary:ent asymmetry}, we generate the nonlinear data using Uniform distribution with different variances. That is, the higher the variance, the more non-Gaussianity. In the controlled experiments we further test our method on different noise variance=$\{1, 2, \textbf{3}, 4, 5\}$. Each algorithm will run at least 10 times.

The results are given in Figure \ref{figure 3}. Overall, our method generally outperforms the baseline methods showing the robustness of our method in variable-varying and variance-varying cases. In detail, in Figure \ref{figure 3a}, we notice that compared with the linear data, all methods require more samples to acquire a decent performance, but our method still outperforms the baseline methods. For our method, 600 is a decent choice for the sample sizes. In figure \ref{figure 3b}, with the growing number of variables, the performance of NOTEARS-MLP decreases rapidly while our method remains stable. The reason is that as the number of variables grows, the edges increase simultaneously, and the incorrectly identified edges will also increase. In the high dimensional case, our method stays the lowest SHD, the lowest FDR, and the highest TPR, indicating that our method is less sensitive to the dimensional increase. In Figure \ref{figure 3c}, we can see that our method is not sensitive to the noise variances while the performance of other methods decreases rapidly as the variances grow. The reason is that the variance controls the noise non-Gaussianity making it far from the standard Gaussian. Therefore, the entropy-based loss is stable and identifiable under any distribution while other methods will lose its identifiability.

\subsection{Real-World Data}

\begin{table}[!t]
\renewcommand{\arraystretch}{1.3}
\caption{Empirical results on Sachs data set}
\label{tab:real_result}
\centering
\begin{tabular}{|c||c||c|}
\hline
Method & Predicting Edges & SHD \\
\hline
NOTEARS & 14 & 12  \\
\hline
LiNGAM & 8 & 15 \\
\hline
GOLEM-NV & 11 & 14 \\
\hline
DAG-GNN & 15 & 16 \\
\hline
NOTEARS-MLP & 16 & 12\\
\hline
GraN-DAG & 10 & 13 \\
\hline
\textbf{ours} & 16 & 13 \\
\hline
\textbf{ours-MLP} & 13 & \textbf{11}\\
\hline
\end{tabular}
\end{table}

Following previous works, we consider the real-world data set for the discovery of a protein signaling network based on expression levels of proteins and phospholipids \cite{sachs2005causal}. This dataset is commonly used for causal structures learning work, with experimental annotations accepted by the biological research community. There are 11 variables and 853 observational samples in the data set, and the ground-truth structure is provided by Sachs et al. \cite{sachs2005causal} containing 17 edges. On this benchmark data set, compared with other methods, $\textbf{ours-MLP}$ achieves the best SHD 11 with 13 estimated edges. It shows that in the real-world situation our method still works better.

\subsection{Algorithm Parameter Settings}

\begin{table*}[!t] 
  \centering
  \caption{Parameter Settings} \label{tab:param_setting}
    \begin{tabular}{llll}
    \toprule
    Parameter & \multicolumn{1}{l}{Symbol} & Value & Applicable to \\
    \midrule
    threshold on $W$ & $\omega$ & 0.3 & ours, ours-MLP, NOTEARS, NOTEARS-MLP, GOLEM-NV \\
    threshold on $W$ & $\omega$ & 0.4 & GraN-DAG \\
    acyclicity penalty & $h(W)$  & $tr\left(e^{W \circ W}\right)-d$ & ours, ours-MLP, NOTEARS, NOTEARS-MLP, GOLEM-NV, GraN-DAG    \\
    acyclicity penalty & $h(W)$ & $tr\left((I+\frac{W \circ W}{d})^d\right)-d$      & DAG-GNN \\
    $h$ tolerance & $\epsilon$      & $10^{-8}$ & ours, ours-MLP, NOTEARS, NOTEARS-MLP, GOLEM-NV, DAG-GNN, GraN-DAG \\
    $h$ progress rate &   $c$    & $0.25$ & ours, ours-MLP, NOTEARS, NOTEARS-MLP, DAG-GNN, GraN-DAG \\
    initial Lagrange multiplier &  $\alpha_0$    & $0$ & ours, ours-MLP, NOTEARS, NOTEARS-MLP, DAG-GNN, GraN-DAG \\
    $\rho$ increase factor &       & 10 & ours, ours-MLP, NOTEARS, NOTEARS-MLP, DAG-GNN, GraN-DAG \\
    $\rho$ maximum &       & $10^{16}$ & ours, ours-MLP, NOTEARS, NOTEARS-MLP \\
    $\rho$ maximum &       & $10^{20}$ & DAG-GNN \\
    optimize algorithms &        & L-BFGS &  ours, ours-MLP, NOTEARS, NOTEARS-MLP    \\
    optimize algorithms &       &  Adam & GOLEM-NV, DAG-GNN \\
    optimize algorithms &       &  SGD & GraN-DAG \\
    seed &       &  $123$ & all \\
    \bottomrule
    \end{tabular}%
  \label{sup tab:addlabel}%
\end{table*}%

The parameter settings of baseline algorithms follow their original papers and codes. Importantly, we use the least-squares loss $ \frac{1}{2m} \| \mathbf{X} - \mathbf{X} W \|^2_F + \lambda \left\|W\right\|_1 $ for NOTEARS \cite{zheng2018dags} and NOTEARS-MLP \cite{zheng2020learning} regardless of the noise type. We use the negative ELBO as the objective function under standard Gaussian noise assumption for DAG-GNN \cite{yu2019dag}. For GOLEM-NV \cite{ng2020role} and GraN-DAG \cite{lachapelle2019gradient}, we use the likelihood-based loss with (multivariate) Gaussian assumption. For \textbf{our} and \textbf{our-MLP}, we use entropy-based loss as the objective function, and the parameter settings for our method are shown in Table \ref{tab:param_setting}. Additionally, the code of NOTEARS and NOTEARS-MLP is available at \url{https://github.com/xunzheng/notears } and the code of DAG-GNN is available at \url{https://github.com/fishmoon1234/DAG-GNN }. The code of GraN-DAG is available at \url{https://github.com/kurowasan/GraN-DAG } and the code of GOLEM-NV is available at \url{https://github.com/ignavierng/golem}. For LiNGAM, we use the causal discovery toolbox packages \cite{cdt}. Our code is available at \url{https://github.com/DMIRLAB-Group/CausalDiscoveryBasedOnEntropy}.

\section{Conclusion}
  In this work, we have re-examined NOTEARS with the least-square loss for learning DAGs in causal discovery. Our analysis shows that NOTEARS with the least-square loss is disabled to discover underlying causal structure under some weak assumptions, and the entropy-based loss is a proper replacement of the least-square loss in NOTEARS framework. We further provide the theoretical justification for the proposed method, by showing the consistency with the likelihood. Our experimental results validate our theoretical analysis and the effectiveness of the proposed entropy-based method. A clear next step is to generalize the theory and algorithms to the more general causal mechanism, e.g., the post-nonlinear causal model.




{\appendices

\setcounter{proposition}{0}
\setcounter{lemma}{0}
\setcounter{theorem}{0}
\setcounter{corollary}{0}
\setcounter{equation}{0}
\numberwithin{equation}{section}

\setcounter{figure}{0}

\setcounter{section}{0}

\section{Proof of Proposition \ref{pro:dist tran}} \label{app_sec:proof of pro1}

We present the property of additive noise model in the following proposition showing that $ p(X, Y) = p(X, N_Y)$ in the causal direction and $p(X, Y)=p(Y,\hat{N}_X)$ in the anti-causal direction.

\begin{proposition}[Noise representation property] \label{pro:dist tran}
For a bivariate model of variables $X_1$ and $X_2$, if the noise $N$ is additive, i.e., $X_2=f(X_1)+N$, then there exists the following distribution transformations:
\begin{equation} \label{eq: pro1_1}
    p(X_1,X_2) = p(X_1,N). 
\end{equation}
If the noise $N$ is independent of $X_1$, we further have:
\begin{equation}  \label{eq: pro1_2}
    p(X_2|X_1) = p(N=X_2-f(X_1)).
\end{equation} 
\end{proposition}

\begin{proof}
For the additive noise model $X_2=f(X_1)+N$, to avoid ambiguity, we consider the transformation from $p(X_1,X_2)$ to $p(X_1',N)$, where we additionally denote $X_1'=X_1$ for understandability. Specifically, the transformation can be expressed as follows:
	\begin{equation} \label{transfromation}
	    \begin{aligned}
    		\text{}\left\{
    		\begin{aligned}
    			X_1' & = X_1,
    			\\ N & = X_2-f(X_1).
    		\end{aligned}
    		\right. 
	    \end{aligned}
	\end{equation}
	Then, the distribution transformations are given as follows: 
	\begin{equation} \label{abs determinant J}
	  \begin{aligned}
		p(X_1,X_2)=p(X_{1}',N)\left|\det(J)\right| = p(X_{1}',N)
	  \end{aligned}
	\end{equation}
	where $J$ denotes the Jacobian matrix of the transformation from $(X_1, X_2)^{T}$ to $(X_{1}', N)^{T}$, and we have:
	\begin{equation} \label{app: eq jocobian value 1}
		\begin{aligned}
			\ \left| \det (J) \right| & =
			\left| \det \begin{pmatrix}
				\frac{\partial X_{1}'}{\partial X_{1}} & \frac{\partial X_{1}'}{\partial X_2}\\
				\frac{\partial N}{\partial X_{1}} & \frac{\partial N}{\partial X_2}
			\end{pmatrix} \right| 
		\\ & =
			\left| \det \begin{pmatrix}
				\frac{\partial X_{1}}{\partial X_{1}} & \frac{\partial X_{1}}{\partial X_2}\\
				\frac{\partial (X_2-f(X_{1}))}{\partial X_{1}} & \frac{\partial (X_2-f(X_{1}))}{\partial X_2}
			\end{pmatrix} \right| 
		\\ & =
			\left| \det \begin{pmatrix}
				1 & 0\\
				-\frac{\partial f(X_{1})}{\partial X_{1}} & 1
			\end{pmatrix} \right| \\ & = 1 ,
		\end{aligned}
	\end{equation}
where the second equality is obtained by substituting Eq. \ref{transfromation} into the first equality.
    
	Hence, substituting Eq. \ref{app: eq jocobian value 1} into Eq. \ref{abs determinant J}, we conclude:
	\begin{equation}
	  \begin{aligned}
	    p(X_1,X_2) = p(X_1', N) = p(X_1, N).
	  \end{aligned}
	\end{equation}
 
     If the noise $N$ is independent of $X_1$, we further have:
		\begin{equation}
	  \begin{aligned}
	    p(X_1)p(X_2|X_1)& = p(X_1)p(N), \\
	  \end{aligned}
	\end{equation}
	and by eliminating $p(X_1)$ in both sides, we have:
		\begin{equation}
	  \begin{aligned}
	    p(X_2|X_1)& =p(N),
	  \end{aligned}
	\end{equation}
	which finishes the proof.
\end{proof}

\section{Proof of Lemma \ref{app: lemma 1} }

In the following lemma, we bridge a connection between the log-likelihood and mutual information for the additive noise model.  The proof of Lemma \ref{app: lemma 1}      is based on the definition of mutual information.
\begin{lemma} \label{app: lemma 1} 
	Given the samples $\left\{x^{( i)} ,y^{( i)}\right\}^{m}_{i=1}$ and causal model $Y=f(X;\theta)+N_{Y} $ with any parameters $\displaystyle \theta $ and as $m \rightarrow \infty$, the average log-likelihood $l_{X\rightarrow Y} (\theta )$ and the mutual information between $X$ and $N_Y$ are related in the following way:
	\begin{equation} \label{eq:mutual information vs likelihood}
		I(X,N_{Y} ;\theta )
		= \mathbb{E}_{x,y\sim p(X,Y)} \left[ p(X=x,Y=y) \right] - \lim_{m \to \infty} l_{X\rightarrow Y} (\theta ),
	\end{equation}
	where
	\begin{equation} \label{eq:general_likelihood}
	  \begin{aligned}
		&l_{X\rightarrow Y} (\theta ) \\
		=& 
		\frac{1}{m} \sum ^{m}_{i=1}\log p(X=x^{( i)} )
		+ \frac{1}{m}\sum ^{m}_{i=1}\log p(N_{Y}=y^{(i)}-f(x^{(i)}) ;\theta ) .
	  \end{aligned}
	\end{equation}
\end{lemma}

\label{app_sec: proof of lemma 1}
\begin{proof}
    The average log-likelihood $l_{X\rightarrow Y} (\theta )$ has the following form:
    \begin{equation} \label{app: eq avg log lih}
        \begin{aligned}
		  & l_{X\rightarrow Y} (\theta ) \\
		= & \frac{1}{m} \sum ^{m}_{i=1}\log p(X=x^{( i)}, Y=y^{( i)} ;\theta) \\
        = &	\frac{1}{m} \sum ^{m}_{i=1}\log p(X=x^{( i)}, N_{Y}=y^{(i)}-f(x^{(i)}) ;\theta ) \\
		= &	\frac{1}{m}	\sum ^{m}_{i=1}\log p(X=x^{( i)} )
		+ \frac{1}{m} \sum ^{m}_{i=1}\log p(N_{Y}=y^{(i)}-f(x^{(i)}) ;\theta ),
	  \end{aligned}
    \end{equation}
    where the second equality is based on Proposition \ref{pro:dist tran} and the third equality is based on the independence between $X$ and $N_Y$.
    
    The mutual information between $X$ and $N_Y$ can be represented by:
    \begin{equation*}
      \begin{aligned}
          & I(X,N_{Y} ;\theta ) \\
        = & \mathbb{E}_{x,y\sim p(X,Y)}\left[ \log \left( \frac{p(X=x,N_Y=y-f(x);\theta)}{ p(X=x)p(N_Y=y-f(x); \theta)} \right) \right]  \\
        = & \mathbb{E}_{x,y\sim p(X,Y)}\left[\log p(X=x,N_Y=y-f(x);\theta) 
        \right. \\ \phantom{=\;\;}
		&\left. -  \log p(X=x) - \log p(N_Y=y-f(x); \theta) \right] 
         \\
        = & \mathbb{E}_{x,y\sim p(X,Y)}\left[\log p(X=x,Y=y))
         \right. \\ \phantom{=\;\;}
		&\left. - \log p(X=x) -\log p(N_Y=y-f(x); \theta) \right] \\ 
        = & \mathbb{E}_{x,y\sim p(X,Y)}\left[\log p(X=x,Y=y)\right] - \lim_{m \to \infty} l_{X\rightarrow Y}(\theta)
      \end{aligned}
    \end{equation*}
    where the third equality is based on Proposition \ref{pro:dist tran} and the fourth equality is obtained by substituting Eq. \ref{app: eq avg log lih}).
\end{proof}

\section{Proof of Lemma \ref{lemma 2}}
\label{app_sec: proof of lemma 2}

In the following lemma, we bridge a connection between the log-likelihood and least-square loss for the additive noise model.
\begin{lemma} \label{lemma 2}
	For the additive noise model $Y=f(X)+N_Y$, as $m \rightarrow \infty$, maximizing the average log-likelihood $l_{X\to Y}(\theta)$ with the standard Gaussian noise assumption,
	\begin{equation} \label{eq:lik_gaussian}
		\begin{aligned}
			l_{X\rightarrow Y} (\theta ) = & \frac{1}{m}
			\sum ^{m}_{i=1}\log\left(\frac{1}{\sqrt{2\pi }}\exp\left( -\frac{(x^{( i)} -0)^{2}}{2}\right)\right) \\
			& + \frac{1}{m} \sum ^{m}_{i=1}\log\left(\frac{1}{\sqrt{2\pi }}\exp\left( -\frac{(y^{( i)} -f (x^{( i)} ;\theta ))^{2}}{2}\right)\right) ,
		\end{aligned}
	\end{equation}
	is equivalent to minimizing the least-square loss $LS_{X \to Y}$,
	\begin{equation} \label{eq:lemma_lsloss}
		LS_{X\rightarrow Y} =\mathbb{E}\left[ (X-0)^{2}\right] +\mathbb{E}\left[ (Y-f (X))^{2}\right] .
	\end{equation}
\end{lemma}
\begin{proof}
	Based on Lemma \ref{app: lemma 1}, as $m \rightarrow \infty$, the average log-likelihood under the standard Gaussian distribution has the following form:
	\begin{equation} \label{eq:gaussian_likelihood}
		\begin{aligned}
			 & \lim_{m \to \infty} l_{X\rightarrow Y} (\theta ) \\
			=&  \lim_{m \to \infty} \left(	\frac{1}{m} \sum ^{m}_{i=1}\log q(X=x^{( i)} )
			\right. \\ \phantom{=\;\;}
		    & \left. + \frac{1}{m}\sum ^{m}_{i=1}\log q(N_{Y}=y^{(i)}-f(x^{(i)}) ;\theta ) \right)\\
		    = & \mathbb{E}_{x,y \sim p(X,Y)} \left[ \log q(X=x) + \log q(N_Y=y-f(X); \theta) \right] \\
		    = &  \mathbb{E}_{x,y \sim p(X,Y)} \left[ 
		    \log\left( \frac{1}{\sqrt{2\pi }}\exp\left( -\frac{(X -0)^{2}}{2}\right)\right)
		    \right. \\ \phantom{=\;\;}
		    & \left. + \log\left( \frac{1}{\sqrt{2\pi }}\exp\left( -\frac{(Y-f(X))^{2}}{2}\right)\right) \right] \\
            = & -\frac{1}{2} \left( \mathbb{E}_{x,y \sim p(X,Y)}\left[ (X-0)^{2} + (Y-f (X))^{2}\right] +const \right) \\
            = & -\frac{1}{2} \left( LS_{X\rightarrow Y} +const \right),
      \end{aligned}
    \end{equation}
    where $q$ is the density of the standard Gaussian distribution, and the third equality is obtained by substituting the standard Gaussian distribution into $q$, and in the fourth equality, $const=\log 4\pi$.
\end{proof}

\section{Proof of Lemma \ref{lemma 3}}
\label{app_sec: proof of lemma 3}

In the following lemma, we bridge a connection between the least-square loss and mutual information for the additive noise model. Based on the following lemma, we can utilize the biased mutual information between the cause and the noise to analyze the least-square loss. The proof of Lemma \ref{lemma 3} is based on Lemma \ref{app: lemma 1} and Lemma \ref{lemma 2}.
\begin{lemma} \label{lemma 3}
	For the additive noise model, as $m \rightarrow \infty$, minimizing the least-square loss is equivalent to minimizing the mutual information under the standard Gaussian noise assumption with the following form:
	\begin{equation} \label{eq:gaussian mutual information}
	 \begin{aligned}
			   I_q(X, &  N_{Y} ;\theta ) 
			  =  \mathbb{E}_{x,y\sim p(X,Y)}\left[\log p(X=x,Y=y) 
			  \right. \\
			  &\left. -\log q(X=x) -\log q(N_Y=y - f(x) ;\theta ) \right] ,
	  \end{aligned}
	\end{equation}
	where $\displaystyle q$ is the density function of standard Gaussian distribution.
\end{lemma}
\begin{proof}
    Let $q$ denote the density function of standard Gaussian distribution. By substituting $q$ into Eq. \ref{eq:mutual information vs likelihood} in Lemma \ref{app: lemma 1}, the mutual information under the standard Gaussian noise assumption can be rewritten as:
    \begin{equation*} 
      \begin{aligned}
         & I_q(X,N_{Y} ;\theta ) \\
		 =&  \mathbb{E}_{x,y\sim p(X,Y)}\left[\log p(X=x,Y=y) \right. \\
		  &\left. -\log q(X=x) -\log q(N_Y=y - f(x) ;\theta ) \right] \\
		= &  \mathbb{E}_{x,y\sim p(X,Y)} \left[ \log p(X=x,Y=y) \right]
		+ \frac{1}{2}\left( LS_{X \rightarrow Y} + const \right),
      \end{aligned}
    \end{equation*}
    where the second equality is based on Eq. \ref{eq:gaussian_likelihood} in Lemma \ref{lemma 2}.
\end{proof}

\section{Proof of Theorem \ref{app: theorem:mse nonlinear}} \label{app_sec: proof of theorem 2}

\setcounter{theorem}{1}
In the following theorem, we establish a certain condition under which the least-square loss can not identify the causal direction, indicating the least-square loss is not a suitable loss for the task of causal discovery. The proof of Theorem \ref{app: theorem:mse nonlinear} is based on Lemma \ref{lemma 3}. 

\begin{theorem} Let $X \to Y$ be the causal direction following the data generation mechanism $Y = f(X) + N_Y$ and we assume $m \rightarrow \infty$. The causal direction is non-identifiable using least-square loss if the following inequality holds: \label{app: theorem:mse nonlinear}
	\begin{equation} \label{app: inequality:mse nonlinear}
		\begin{aligned}
			-\int p(X)\log q(X)dX -\int p(N_{Y})\log q(N_{Y})dN_{Y}  \\
			> -\int p(Y)\log q(Y)dY -\int p(\hat{N}_{X} )\log q(\hat{N}_{X} )d\hat{N}_{X},
		\end{aligned}
	\end{equation}
	where q is the density function of standard Gaussian distribution.
\end{theorem}

\begin{proof}
	Without loss of generality, we assume the identifiability condition of the additive noise model holds. Then, we have $I(X,N_{Y} )< I(Y,\hat{N}_{X} )$. However, based on Eq. \ref{eq:gaussian mutual information} of Lemma \ref{lemma 3} as $m \rightarrow \infty$, using the least-square loss, the mutual information becomes:
	\begin{equation*}  
		\begin{aligned}
		     I_q(X ,  N_{Y}  ) 
			  =&   \mathbb{E}_{x,y\sim p(X,Y)}\left[\log p(X=x,Y=y) 
			  \right. \\
			  &\left. -\log q(X=x) -\log q(N_Y=y - f(x)  ) \right] ,\\
		    I_q(Y , \hat{N}_{X}  )  
			  = & \mathbb{E}_{x,y\sim p(X,Y)}\left[\log p(X=x,Y=y) 
			  \right. \\
			  &\left. -\log q(Y=x) -\log q(\hat{N}_{X}=x - g(y)  ) \right] , 
		\end{aligned}
	\end{equation*}
	where $q\sim N(0,\mathbf{I})$. In this case, the inequality $\displaystyle I_{q} (X,N_{Y} )< I_{q} (Y,\hat{N}_{X} )$ does not necessarily hold. In fact, by solving the inequality $\displaystyle I_{q} (X,N_{Y} ) >I_{q} (Y,\hat{N}_{X} )$, we obtain :
	\begin{equation}\label{app: theorem2_proof_eq:MI}
	    \begin{aligned}
	      \mathbb{E}_{p(X)}\left[ -\log q(X) \right] + \mathbb{E}_{p(X,Y)}\left[ -\log q(N_Y) \right] \\ >
	      \mathbb{E}_{p(Y)}\left[ -\log q(X)  \right] + \mathbb{E}_{p(X,Y)}\left[ -\log q(\hat{N}_X) \right] .
	    \end{aligned}
	\end{equation}
	Note that, the expectation of $N_Y$ can be rewritten as: 
	\begin{equation} \label{app: theorem2_proof_eq:qny}
	    \begin{aligned}
              \mathbb{E}_{ p(X,Y)}\left[ -\log q(N_Y) \right]  
             =& \mathbb{E}_{ p(X,N_Y)}\left[ -\log q(N_Y  ) \right] \\
             =& \mathbb{E}_{ p(N_Y)}\left[ -\log q(N_Y  ) \right] ,
	    \end{aligned}
	\end{equation}
	in which the first equality holds due to Proposition \ref{pro:dist tran} (see Appendix \ref{app_sec:proof of pro1}), and the second equality is obtained by integrating $X$.
	Similarly, we have 
	\begin{equation} \label{app: theorem2_proof_eq:qnx}
	    \mathbb{E}_{ p(X,Y)}\left[ -\log q(\hat{N}_X) \right] = \mathbb{E}_{p(\hat{N}_X)}\left[ -\log q(\hat{N}_X) \right].
	\end{equation}
	By substituting Eq. \ref{app: theorem2_proof_eq:qny} and Eq. \ref{app: theorem2_proof_eq:qnx} into inequality \ref{app: theorem2_proof_eq:MI}, we obtain inequality \ref{app: inequality:mse nonlinear}
\end{proof}

\section{Proof of Theorem \ref{app theorem entropy vs likelihood}}
\label{app_sec: proof of theorem 3}

\setcounter{theorem}{2}

In the following theorem, we show that, for the additive noise model, maximizing the log-likelihood score is equivalent to maximizing the entropy of noise. It inspires our entropy loss-based method for causal discovery.
\begin{theorem} \label{app theorem entropy vs likelihood}
	In the additive noise model, the entropy-based score has a consistency with the log-likelihood score when the sample sizes $m \rightarrow \infty $, i.e.,
	\begin{equation}
		\lim_{m \to \infty}\frac{1}{m}\sum ^{m}_{j=1}\sum ^{d}_{i=1}\log p\left( x^{(j)}_{i} |x^{( j)}_{pa( i)}\right) =-\sum ^{d}_{i=1} H(N_{i} ).
	\end{equation}
\end{theorem}

\begin{proof}
	Based on Lemma \ref{app: lemma 1} we generalize the log-likelihood from two variables case to the multivariate case using Proposition \ref{pro:dist tran} as follows:
	\begin{equation*} 
		\begin{aligned}
			& \frac{1}{m}\sum ^{m}_{j=1}\sum ^{d}_{i=1}\log p\left( x^{(j)}_{i} |x^{( j)}_{pa( i)}\right) \\
			 = &	\frac{1}{m}\sum ^{m}_{j=1}\sum ^{d}_{i=1}\log p\left( N_{i} =x^{( j)}_{pa( i)} - f_i(x^{(j)}_{i}) \right).
		\end{aligned}
	\end{equation*}
	Then, as $m \rightarrow \infty$, we have:
		\begin{equation*}
		\begin{aligned}
			& 	\lim_{m \to \infty} \frac{1}{m}\sum ^{m}_{j=1}\sum ^{d}_{i=1}\log p\left( N_{i} =x^{( j)}_{pa( i)} - f_i(x^{(j)}_{i}) \right) \\
			& = - \sum ^{d}_{i=1} \mathbb{E} \left[ - \log p\left( N_{i} \right) \right] \\
			& = - \sum ^{d}_{i=1} H(N_{i} ).
		\end{aligned}
	\end{equation*}
\end{proof}

\section{Proof of Corollary \ref{app corollary:ent asymmetry}}
\label{app_sec: proof of corollary 1}

In the following corollary we show that, for the additive noise model, the entropy-based loss can identify the causal direction correctly. It promises the correctness of our method. The proof of Corollary \ref{app corollary:ent asymmetry} is based on the independence property of ANM, i.e., $I(X, N_Y)<I(Y,\hat{N}_X)$. 
\begin{corollary} \label{app corollary:ent asymmetry}
	For each pair of additive noise model $Y=f(X)+N_Y$, if $v^{\prime\prime}(Y-f(X))f^{\prime}(X) \neq 0$ and the condition in Eq. \ref{eq:anm condition}  does not hold, then using the entropy-based loss each pair of additive noise model is identifiable and the following inequality holds: 
	\begin{equation} \label{app inequality: entropy}
		H\left(X\right) + H\left(N_{Y}\right) < H\left(Y\right) + H\left(\hat{N}_{X}\right).
	\end{equation}
\end{corollary}

\begin{proof}
To prove corollary \ref{app corollary:ent asymmetry}, we will compare the mutual information between the hypothetical cause and the regression residual in causal and anti-causal directions. Based on Lemma 4, if Eq. \ref{eq:anm condition}  does not hold, i.e., no backward ANM model exists, then we must have the following equation:
\begin{equation*}
    \begin{aligned}
       p(X,Y) =  p(N_Y)p(X), \\
       p(X,Y) \neq  p(\hat{N}_X)p(Y),
    \end{aligned}
\end{equation*}
from the causal direction and anti-causal direction, respectively. It means that the noise and the hypothetical cause in the anti-causal direction are not independent of each other, i.e., $N_Y  \Vbar X $ in the causal direction but $\hat{N}_X \nVbar Y$ in the anti-causal direction. Such a property can be represented by mutual information as follows:
\begin{equation*} 
	\begin{aligned}
	    I(X, N_Y)  = 0 \quad  and \quad  I(Y, \hat{N}_X) > 0 \\
	\end{aligned}
\end{equation*}	
Hence $I(X, N_Y) < I(Y, \hat{N}_X) $, i.e.,
\begin{equation*} 
  \begin{aligned}
	\int p(X,N_Y) \log \frac{p(X,N_Y)}{p(X)p(N_Y)} dXdN_Y 
	\\ < \int p(Y,\hat{N}_X) \log \frac{p(Y,\hat{N}_X)}{p(Y)p(\hat{N}_X)} dYd\hat{N}_X .
  \end{aligned}
\end{equation*}
By rewriting the inequality above using entropy, we have:
	\begin{equation} \label{app MI compare}
		\begin{aligned}
			  H\left(X\right) + H\left(N_{Y}\right) +\mathbb{E}\log p(X,N_Y) \\
			   < H\left(Y\right) + H\left(\hat{N}_{X}\right) + \mathbb{E}\log p(Y,\hat{N}_X).
		\end{aligned}
	\end{equation}
	By utilizing Proposition \ref{pro:dist tran}, we have
	\begin{equation}\label{app eq:equal}
	 p(X,Y)=p(X,N_Y)=p(Y,\hat{N}_X).
	\end{equation}
	By substituting Eq. \ref{app eq:equal} into Eq. \ref{app MI compare}, we obtain:
    \begin{equation*}
		\begin{aligned}
			  H\left(X\right) + H\left(N_{Y}\right) +\mathbb{E}\log p(X,Y) \\
			   < H\left(Y\right) + H\left(\hat{N}_{X}\right) + \mathbb{E}\log p(X.Y),
		\end{aligned}
	\end{equation*}
    and by eliminating the same term of both sides, we obtain:
	\begin{equation*}
	    H(X) + H(N_Y) < H(Y) + H(\hat{N}_X),
	\end{equation*}
	which finishes the proof.
\end{proof}

\section{Additional Experimental Details}

For a better presentation of the experimental results, we provide the precise values of Figure \ref{figure 2} and \ref{figure 3} in Table \ref{tab:table_result1} and \ref{tab:table_result2} respectively.

\begin{table*}[!ht]
\caption{Precise Results on the Synthetic Linear Data}
\label{tab:table_result1}
\begin{center}
\begin{tabular}{c|c|c|ccccc} 
    \hline
	\multirow{2}{*}{Metric} & \multirow{2}{*}{Control Param} & \multirow{2}{*}{Value} & ours & NOTEARS & GOLEM-NV   & DAG-GNN   & LiNGAM        \\ 
	\cline{4-8}
	 & \ & \  & mean  $_{\pm \text{ std}}$ & mean   $_{\pm \text{ std}}$  & mean   $_{\pm \text{ std}}$    & mean    $_{\pm \text{ std}}$   & mean $_{\pm \text{ std}}$   \\ 
	\hline
	 
	\multirow{17}{*}{SHD} 
	 & \multirow{5}{*}{sample size} 
	 & 200 
	    & {8.2333} $_{\pm \text{ 5.2197}}$ & 14.2000 $_{\pm \text{ 4.9288}}$ & 12.0000 $_{\pm \text{ 5.4160}}$ & 29.7333 $_{\pm \text{ 5.1376}}$ & 18.6666 $_{\pm \text{ 7.4490}}$   \\ 
	 & & 400   
	    & 7.0333 $_{\pm \text{ 7.6397}}$ & 13.3000 $_{\pm \text{ 5.1778}}$ & 10.1333 $_{\pm \text{ 5.6905}}$ & 23.0000 $_{\pm \text{ 6.6882}}$ & {2.7000} $_{\pm \text{ 4.5981}}$   \\ 
	 & & 600   
	    & 6.0333 $_{\pm \text{ 6.7353}}$ & 12.9333 $_{\pm \text{ 5.3287}}$ & 9.7333 $_{\pm \text{ 5.3287}}$ & 20.3333 $_{\pm \text{ 5.8783}}$ & {0.9000} $_{\pm \text{ 1.1357}}$   \\ 
	 & & 800   
	    & 4.8666 $_{\pm \text{ 4.3568}}$ & 12.5333 $_{\pm \text{ 5.2582}}$ & 10.6333 $_{\pm \text{ 6.6557}}$ & 20.6000 $_{\pm \text{ 6.0310}}$ & {0.8333} $_{\pm \text{ 1.0027}}$   \\ 
	 & & 1000   
	    & 3.8000 $_{\pm \text{ 4.6000}}$ & 15.3666 $_{\pm \text{ 6.8141}}$ & 13.1333 $_{\pm \text{ 6.3809}}$ & 21.1333 $_{\pm \text{ 5.8636}}$ & {0.7666} $_{\pm \text{ 0.7608}}$   \\ 
	 \cline{2-8}
	 & \multirow{7}{*}{number of variable} 
	 & 5  
	    & {0.6333} $_{\pm \text{ 1.6856}}$ & 2.8333 $_{\pm \text{ 1.8454}}$ & 1.5333 $_{\pm \text{ 1.3597}}$ & 7.7666 $_{\pm \text{ 2.0765}}$ & 0.0000 $_{\pm \text{ 0.0000}}$   \\ 
	 & & 10   
	    & {2.5666} $_{\pm \text{ 4.0880}}$ & 8.6666 $_{\pm \text{ 3.8844}}$ & 6.3667 $_{\pm \text{ 3.8426}}$ & 13.4333 $_{\pm \text{ 3.9045}}$ & 0.5666 $_{\pm \text{ 0.6674}}$   \\  
	 & & 15   
	 	& {2.5666} $_{\pm \text{ 3.0406}}$ & 12.8666 $_{\pm \text{ 5.1622}}$ & 9.7333 $_{\pm \text{ 5.3287}}$ & 20.3333 $_{\pm \text{ 5.8783}}$ &  0.9000 $_{\pm \text{ 1.1357}}$   \\ 
	 & & 20   
	 	& {3.5666} $_{\pm \text{ 4.0635}}$ & 12.9666 $_{\pm \text{ 7.8165}}$ & 13.9667 $_{\pm \text{ 6.9545}}$ & 30.0333 $_{\pm \text{ 11.2412}}$ & 1.5666 $_{\pm \text{ 1.2297}}$   \\ 
	 & & 25   
	 	& {5.1000} $_{\pm \text{ 3.9778}}$ & 18.4333 $_{\pm \text{ 7.5351}}$ & 19.8333 $_{\pm \text{ 7.9920}}$ & 33.6333 $_{\pm \text{ 7.4228}}$ & 3.9333 $_{\pm \text{ 3.8981}}$   \\ 
	 & & 50   
	 	& {7.5666} $_{\pm \text{ 4.8695}}$ &  20.9000 $_{\pm \text{ 7.0349}}$ & 40.4667 $_{\pm \text{ 9.1423}}$ & 62.6666 $_{\pm \text{ 9.4210}}$ & 105.9000 $_{\pm \text{ 17.6943}}$   \\ 
	 & & 100   
	 	& {15.9000} $_{\pm \text{ 7.5026}}$ & 35.4333 $_{\pm \text{ 11.7776}}$ & 92.2000 $_{\pm \text{ 19.0725}}$ & 135.7000 $_{\pm \text{34.8072}}$ & 216.1666 $_{\pm \text{ 21.0223}}$   \\ 
	 \cline{2-8}
	 & \multirow{7}{*}{number of variable} 
	 & 5  
	    & {0.0000} $_{\pm \text{ 0.0000}}$ & 2.5333 $_{\pm \text{ 1.5648}}$ & 1.5333 $_{\pm \text{ 1.4545}}$ & 5.6333 $_{\pm \text{ 2.0893}}$ & 0.2000 $_{\pm \text{ 0.5416}}$   \\ 
	 & & 10   
	    & {0.8333} $_{\pm \text{ 2.5309}}$ & 8.5666 $_{\pm \text{ 3.4125}}$ & 6.4000 $_{\pm \text{ 4.0050}}$ & 14.7666 $_{\pm \text{ 2.7164}}$ & 7.6000 $_{\pm \text{ 6.0915}}$   \\  
	 & & 15   
	 	& {0.1666} $_{\pm \text{ 0.4533}}$ & 14.3333 $_{\pm \text{ 5.0684}}$ & 9.7667 $_{\pm \text{ 4.5511}}$ & 18.0666 $_{\pm \text{ 4.7674}}$ & 21.1666 $_{\pm \text{ 11.0305}}$   \\ 
	 & & 20   
	 	& {1.4333} $_{\pm \text{ 1.4067}}$ & 16.6666 $_{\pm \text{ 8.8894}}$ & 14.0000 $_{\pm \text{ 7.8655}}$ & 22.8666 $_{\pm \text{ 7.8133}}$ & 36.9000 $_{\pm \text{ 15.2738}}$   \\ 
	 & & 25   
	 	& {0.4000} $_{\pm \text{0.6110}}$ & 23.0000 $_{\pm \text{  7.4072}}$ & 19.1333 $_{\pm \text{ 7.0462}}$ & 34.4333 $_{\pm \text{ 7.6056}}$ & 52.0666 $_{\pm \text{ 16.1883}}$   \\ 
	 & & 50   
	 	& {1.4000} $_{\pm \text{ 2.1228}}$ & 23.4000 $_{\pm \text{ 9.0428}}$ & 41.8000 $_{\pm \text{ 8.2438}}$ & 57.4666 $_{\pm \text{ 11.2241}}$ & 125.2666 $_{\pm \text{ 23.4036}}$   \\ 
	 & & 100   
	 	& {4.4000} $_{\pm \text{ 4.4766}}$ & 37.4666 $_{\pm \text{ 10.5158}}$ & 91.9000 $_{\pm \text{ 16.7120}}$ & 163.5000 $_{\pm \text{ 64.6151}}$ & 235.3000 $_{\pm \text{ 19.6793}}$   \\ 
	 \hline 
	
	\multirow{17}{*}{FDR} 
	 	 & \multirow{5}{*}{sample size} 
	 & 200 
	    & {0.1953} $_{\pm \text{ 0.1145}}$ & 0.3230 $_{\pm \text{ 0.0989}}$ & 0.2758 $_{\pm \text{ 0.1146}}$ & 0.2502 $_{\pm \text{ 0.4185}}$ & 0.3401 $_{\pm \text{ 0.2229}}$   \\ 
	 & & 400   
	    & {0.1586} $_{\pm \text{ 0.1560}}$ & 0.3093 $_{\pm \text{ 0.1045}}$ & 0.2470 $_{\pm \text{ 0.1244}}$ & 0.4269 $_{\pm \text{ 0.2042}}$ & 0.0665 $_{\pm \text{ 0.1058}}$   \\ 
	 & & 600   
	    & {0.1373} $_{\pm \text{ 0.1425}}$ & 0.3012 $_{\pm \text{ 0.1117}}$ & 0.2370 $_{\pm \text{ 0.1196}}$ & 0.3150 $_{\pm \text{  0.1542}}$ & 0.0258 $_{\pm \text{  0.0337}}$   \\ 
	 & & 800   
	    &  {0.1208} $_{\pm \text{ 0.1022}}$ & 0.2860 $_{\pm \text{ 0.1192}}$ & 0.2497 $_{\pm \text{ 0.1465}}$ & 0.3524 $_{\pm \text{ 0.1720}}$ & 0.0250 $_{\pm \text{  0.0297}}$   \\ 
	 & & 1000   
	    & {0.0836} $_{\pm \text{ 0.0974}}$ & 0.3454 $_{\pm \text{ 0.1363}}$ & 0.3063 $_{\pm \text{ 0.1280}}$ & 0.3239 $_{\pm \text{ 0.1757}}$ & 0.0243 $_{\pm \text{ 0.0236}}$   \\ 
	 \cline{2-8}
	 & \multirow{7}{*}{number of variable} 
	 & 5  
	    & {0.0342} $_{\pm \text{ 0.0589}}$ & 0.1634 $_{\pm \text{ 0.1173}}$ & 0.0997 $_{\pm \text{ 0.0864}}$ & 0.1734 $_{\pm \text{0.3153}}$ & 0.0000 $_{\pm \text{ 0.0000}}$   \\ 
	 & & 10   
	    & {0.0776} $_{\pm \text{ 0.1401}}$ & 0.2865 $_{\pm \text{ 0.1346}}$ & 0.2127 $_{\pm \text{ 0.1268}}$ & 0.3322 $_{\pm \text{ 0.1607}}$ & 0.0251 $_{\pm \text{ 0.0288}}$   \\  
	 & & 15   
	 	& {0.0447} $_{\pm \text{ 0.0650}}$ & 0.3055 $_{\pm \text{ 0.1046}}$ & 0.2370 $_{\pm \text{ 0.1196}}$ & 0.3150 $_{\pm \text{ 0.1542}}$ & 0.0258 $_{\pm \text{ 0.0337}}$   \\ 
	 & & 20   
	 	& {0.0527} $_{\pm \text{0.0662}}$ & 0.2263 $_{\pm \text{ 0.1286}}$ & 0.2528 $_{\pm \text{ 0.1078}}$ & 0.3788 $_{\pm \text{ 0.3075}}$ & 0.0333 $_{\pm \text{ 0.0254}}$   \\ 
	 & & 25   
	 	& {0.0521} $_{\pm \text{ 0.0500}}$ & 0.2596 $_{\pm \text{ 0.0952}}$ & 0.2821 $_{\pm \text{ 0.0992}}$ & 0.3594 $_{\pm \text{ 0.1523}}$ & 0.0624 $_{\pm \text{ 0.0578}}$   \\ 
	 & & 50   
	 	& {0.0392} $_{\pm \text{ 0.0301}}$ & 0.1293 $_{\pm \text{ 0.0439}}$ & 0.2857 $_{\pm \text{ 0.0546}}$ & 0.3526 $_{\pm \text{ 0.0634}}$ & 0.5957 $_{\pm \text{ 0.0595}}$   \\ 
	 & & 100   
	 	& {0.0365} $_{\pm \text{ 0.0185}}$ & 0.1100 $_{\pm \text{ 0.0354}}$ & 0.3231 $_{\pm \text{ 0.0521}}$ & 0.3929 $_{\pm \text{ 0.1637}}$ & 0.5994 $_{\pm \text{ 0.0411}}$   \\ 
	 \cline{2-8}
	 & \multirow{7}{*}{number of variable} 
	 & 5  
	    & {0.0000} $_{\pm \text{ 0.0000}}$ & 0.1583 $_{\pm \text{ 0.1042}}$ & 0.1005 $_{\pm \text{ 0.0963}}$ & 0.3453 $_{\pm \text{ 0.2403}}$ & 0.0200 $_{\pm \text{ 0.0541}}$   \\ 
	 & & 10   
	    & {0.0223} $_{\pm \text{ 0.0703}}$ & 0.2875 $_{\pm \text{ 0.1037}}$ & 0.2161 $_{\pm \text{ 0.1287}}$ &  0.3966 $_{\pm \text{ 0.1426}}$ & 0.2727 $_{\pm \text{ 0.1906}}$   \\  
	 & & 15   
	 	& {0.0000} $_{\pm \text{ 0.0000}}$ & 0.3403 $_{\pm \text{ 0.1046}}$ & 0.2385 $_{\pm \text{ 0.1108}}$ & 0.3074 $_{\pm \text{ 0.1237}}$ & 0.4451 $_{\pm \text{ 0.2000}}$   \\ 
	 & & 20   
	 	& {0.0000} $_{\pm \text{ 0.0000}}$ & 0.2938 $_{\pm \text{ 0.1327}}$ & 0.2166 $_{\pm \text{ 0.1197}}$ & 0.3203 $_{\pm \text{ 0.1353}}$ & 0.5505 $_{\pm \text{ 0.1674}}$   \\ 
	 & & 25   
	 	& {0.0006} $_{\pm \text{ 0.0035}}$ & 0.3250 $_{\pm \text{ 0.0873}}$ & 0.2713 $_{\pm \text{ 0.0903}}$ & 0.3729 $_{\pm \text{ 0.1224}}$ & 0.5962 $_{\pm \text{ 0.1170}}$   \\ 
	 & & 50   
	 	& {0.0050} $_{\pm \text{ 0.0106}}$ & 0.1457 $_{\pm \text{ 0.0560}}$ & 0.3004 $_{\pm \text{ 0.0495}}$ & 0.3159 $_{\pm \text{ 0.0704}}$ &  0.6680 $_{\pm \text{ 0.0746}}$   \\ 
	 & & 100   
	 	& {0.0123} $_{\pm \text{ 0.0148}}$ & 0.1154 $_{\pm \text{ 0.0341}}$ & 0.3111 $_{\pm \text{ 0.0515}}$ & 0.4559 $_{\pm \text{ 0.1661}}$ &  0.6460 $_{\pm \text{ 0.0375}}$   \\ 
	 \hline 
	
	\multirow{17}{*}{TPR} 
	 & \multirow{5}{*}{sample size} 
	 & 200 
	    & {0.8488} $_{\pm \text{ 0.1137}}$ & 0.7122 $_{\pm \text{0.1100}}$ & 0.7622 $_{\pm \text{ 0.1141}}$ & 0.0322 $_{\pm \text{ 0.1552}}$ & 0.5255 $_{\pm \text{ 0.1562}}$   \\ 
	 & & 400   
	    & {0.8777} $_{\pm \text{ 0.1396}}$ & 0.7222 $_{\pm \text{  0.1245}}$ & 0.8000 $_{\pm \text{  0.1054}}$ & 0.3655 $_{\pm \text{ 0.1804}}$ & 0.9655 $_{\pm \text{ 0.0892}}$   \\ 
	 & & 600   
	    & {0.9011} $_{\pm \text{ 0.0737}}$ & 0.7544 $_{\pm \text{ 0.0983}}$ & 0.8278 $_{\pm \text{ 0.0988}}$ & 0.4255 $_{\pm \text{ 0.1939}}$ & 0.9977 $_{\pm \text{  0.0083}}$   \\ 
	 & & 800   
	    & {0.9266} $_{\pm \text{ 0.0701}}$ & 0.7633 $_{\pm \text{ 0.0963}}$ & 0.8056 $_{\pm \text{ 0.1244}}$ & 0.4455 $_{\pm \text{ 0.1588}}$ & 0.9988 $_{\pm \text{ 0.0059}}$   \\ 
	 & & 1000   
	    & {0.9144} $_{\pm \text{0.0980}}$ & 0.6944 $_{\pm \text{ 0.1285}}$ & 0.7689 $_{\pm \text{ 0.0921}}$ & 0.3866 $_{\pm \text{ 0.1567}}$ & 1.0000 $_{\pm \text{ 0.0000}}$   \\ 
	 \cline{2-8}
	 & \multirow{7}{*}{number of variable} 
	 & 5  
	    & {0.9366} $_{\pm \text{ 0.1168}}$ & 0.7166 $_{\pm \text{ 0.1845}}$ & 0.8467 $_{\pm \text{ 0.1360}}$ & 0.2233 $_{\pm \text{ 0.2076}}$ & 1.0000 $_{\pm \text{ 0.0000}}$   \\ 
	 & & 10   
	    & {0.9016} $_{\pm \text{ 0.1546}}$ & 0.7183 $_{\pm \text{ 0.1228}}$ & 0.7933 $_{\pm \text{ 0.1086}}$ & 0.4366 $_{\pm \text{ 0.1682}}$ & 0.9983 $_{\pm \text{ 0.0089}}$   \\  
	 & & 15   
	 	& {0.9333} $_{\pm \text{ 0.0730}}$ & 0.7655 $_{\pm \text{ 0.0948}}$ & 0.8278 $_{\pm \text{ 0.0985}}$ & 0.4255 $_{\pm \text{ 0.1939}}$ & 0.9977 $_{\pm \text{ 0.0083}}$   \\ 
	 & & 20   
	 	& {0.9324} $_{\pm \text{ 0.0704}}$ & 0.7875 $_{\pm \text{ 0.1069}}$ & 0.7992 $_{\pm \text{ 0.1161}}$ & 0.3591 $_{\pm \text{ 0.2793}}$ & 0.9958 $_{\pm \text{  0.0113}}$   \\ 
	 & & 25   
	 	& {0.9173} $_{\pm \text{ 0.0555}}$ & 0.7760 $_{\pm \text{ 0.0844}}$ & 0.7852 $_{\pm \text{ 0.0881}}$ & 0.4473 $_{\pm \text{ 0.1411}}$ & 0.9873 $_{\pm \text{ 0.0222}}$   \\ 
	 & & 50   
	 	& {0.9426} $_{\pm \text{ 0.0352}}$ & 0.8370 $_{\pm \text{ 0.0571}}$ & 0.7770 $_{\pm \text{ 0.0755}}$ & 0.5130 $_{\pm \text{ 0.0985}}$ & 0.6140 $_{\pm \text{ 0.0635}}$   \\ 
	 & & 100   
	 	& {0.9351} $_{\pm \text{ 0.0280}}$ & 0.8646 $_{\pm \text{ 0.0423}}$ & 0.7768 $_{\pm \text{ 0.0689}}$ & 0.4200 $_{\pm \text{ 0.1813}}$ & 0.5623 $_{\pm \text{ 0.0420}}$   \\ 
	 \cline{2-8}
	 & \multirow{7}{*}{number of variable} 
	 & 5  
	    & {1.0000} $_{\pm \text{ 0.0000}}$ & 0.7466 $_{\pm \text{ 0.1564}}$ & 0.8467 $_{\pm \text{ 0.1454}}$ & 0.4366 $_{\pm \text{ 0.2089}}$ & 0.9800 $_{\pm \text{ 0.0541}}$   \\ 
	 & & 10   
	    & {0.9666} $_{\pm \text{ 0.0933}}$ & 0.7150 $_{\pm \text{ 0.1057}}$ & 0.7950 $_{\pm \text{ 0.1127}}$ & 0.4100 $_{\pm \text{ 0.1540}}$ & 0.8283 $_{\pm \text{ 0.1492}}$   \\  
	 & & 15   
	 	& {0.9944} $_{\pm \text{ 0.0151}}$ & 0.7400 $_{\pm \text{ 0.0908}}$ & 0.8300 $_{\pm \text{ 0.0845}}$ & 0.5110 $_{\pm \text{ 0.1698}}$ & 0.7211 $_{\pm \text{ 0.1471}}$   \\ 
	 & & 20   
	 	& {0.9641} $_{\pm \text{ 0.0351}}$ & 0.7616 $_{\pm \text{ 0.1198}}$ & 0.7975 $_{\pm \text{ 0.1224}}$ & 0.5541 $_{\pm \text{ 0.1773}}$ &  0.6450 $_{\pm \text{ 0.1616}}$   \\ 
	 & & 25   
	 	& {0.9920} $_{\pm \text{ 0.0122}}$ & 0.7433 $_{\pm \text{ 0.0824}}$ & 0.7907 $_{\pm \text{ 0.0744}}$ & 0.4326 $_{\pm \text{ 0.1540}}$ & 0.6166 $_{\pm \text{ 0.1131}}$   \\ 
	 & & 50   
	 	& {0.9886} $_{\pm \text{ 0.0145}}$ & 0.8193 $_{\pm \text{ 0.0744}}$ & 0.7827 $_{\pm \text{0.0582}}$ & 0.5556 $_{\pm \text{ 0.1192}}$ & 0.5263 $_{\pm \text{ 0.0695}}$   \\ 
	 & & 100   
	 	& {0.9849} $_{\pm \text{ 0.0136}}$ & 0.8583 $_{\pm \text{ 0.0389}}$ & 0.7610 $_{\pm \text{ 0.0711}}$ & 0.4091 $_{\pm \text{ 0.2030}}$ & 0.5020 $_{\pm \text{ 0.0446}}$   \\ 
	 \hline 
\end{tabular}
\end{center}
\end{table*}

\begin{table*}[!ht]
\caption{Precise Results on the Synthetic Nonlinear Data}
\label{tab:table_result2}
\begin{center}

\begin{tabular}{c|c|c|cccc} 
    \hline
	\multirow{2}{*}{Metric} & \multirow{2}{*}{Control Param} & \multirow{2}{*}{Value} & ours-MLP & NOTEARS-MLP & DAG-GNN   & GranDAG           \\ 
	\cline{4-7}
	 & \ & \  & mean  $_{\pm \text{ std}}$ & mean   $_{\pm \text{ std}}$  & mean   $_{\pm \text{ std}}$    & mean    $_{\pm \text{ std}}$    \\ 
	\hline
	 
	\multirow{17}{*}{SHD} 
	 & \multirow{5}{*}{sample size} 
	 & 200 
	    & {16.7000} $_{\pm \text{ 6.1814}}$ & 29.1000 $_{\pm \text{ 5.1662}}$ & 25.3000 $_{\pm \text{ 3.9509}}$ & 43.0000 $_{\pm \text{ 4.1231}}$    \\ 
	 & & 400   
	    & {8.0000} $_{\pm \text{ 1.9493}}$ & 24.5000 $_{\pm \text{ 3.2634}}$ & 20.9000 $_{\pm \text{ 2.9137}}$ & 40.4000 $_{\pm \text{ 6.7705}}$    \\ 
	 & & 600   
	    & {5.1000} $_{\pm \text{ 2.2113}}$ & 15.0000 $_{\pm \text{ 4.4045}}$ & 22.3000 $_{\pm \text{ 4.5617}}$ & 38.4000 $_{\pm \text{ 5.9363}}$    \\ 
	 & & 800   
	    & {4.6000} $_{\pm \text{ 1.6852}}$ & 10.8000 $_{\pm \text{ 3.2496}}$ & 22.5000 $_{\pm \text{ 2.9410}}$ & 39.0000 $_{\pm \text{ 4.8785}}$    \\ 
	 & & 1000   
	    & {1.8000} $_{\pm \text{ 1.1661}}$ & 10.3000 $_{\pm \text{ 6.0671}}$ & 14.2000 $_{\pm \text{ 4.9558}}$ & 36.3000 $_{\pm \text{ 3.5791}}$    \\ 
	 \cline{2-7}
	 & \multirow{7}{*}{number of variable} 
	 & 5  
	    & {0.3000} $_{\pm \text{ 0.4582}}$ & 2.3000 $_{\pm \text{ 2.2383}}$ & 6.9000 $_{\pm \text{ 1.2206}}$ & 6.8000 $_{\pm \text{ 2.2271}}$    \\ 
	 & & 10   
	    & {2.2000} $_{\pm \text{ 1.4696}}$ & 6.2000 $_{\pm \text{ 3.4292}}$ & 15.4000 $_{\pm \text{ 3.6110}}$ & 26.6000 $_{\pm \text{ 3.2000}}$    \\ 
	 & & 15   
	    & {5.1000} $_{\pm \text{ 2.2113}}$ & 15.0000 $_{\pm \text{ 4.4045}}$ & 22.3000 $_{\pm \text{ 4.5617}}$ & 38.4000 $_{\pm \text{ 5.9363}}$    \\ 
	 & & 20   
	    & {7.7000} $_{\pm \text{ 2.9681}}$ & 24.7000 $_{\pm \text{ 4.7339}}$ & 25.7000 $_{\pm \text{ 3.7960}}$ & 64.4000 $_{\pm \text{ 12.0432}}$    \\ 
	 & & 25   
	    & {8.4000} $_{\pm \text{ 3.7469}}$ & 47.5000 $_{\pm \text{ 8.8797}}$ & 29.7000 $_{\pm \text{ 3.5227}}$ & 73.9000 $_{\pm \text{ 8.6769}}$    \\ 
	 & & 50   
	    & {40.0000} $_{\pm \text{ 4.5825}}$ & 77.4000 $_{\pm \text{ 7.2828}}$ & 60.5000 $_{\pm \text{ 6.3757}}$ & 176.2000 $_{\pm \text{ 36.3670}}$    \\ 
	 & & 100   
	    & {113.4000} $_{\pm \text{ 6.6813}}$ & 143.7000 $_{\pm \text{ 10.4216}}$ & 137.8000 $_{\pm \text{ 17.2846}}$ & 339.5000 $_{\pm \text{ 22.9183}}$    \\  
	 \cline{2-7}
	 & \multirow{5}{*}{noise variance} 
	 & 1  
	    & {5.5000} $_{\pm \text{ 2.4186}}$ & 5.6000 $_{\pm \text{ 2.4979}}$ & 23.8000 $_{\pm \text{ 5.8446}}$ & 29.7000 $_{\pm \text{ 3.0675}}$    \\ 
	 & & 2   
	    & {4.5000} $_{\pm \text{ 2.2472}}$ & 11.3000 $_{\pm \text{ 8.2589}}$ & 20.3000 $_{\pm \text{ 1.6155}}$ & 30.3000 $_{\pm \text{2.4515}}$    \\ 
	 & & 3   
	    & {5.1000} $_{\pm \text{ 2.2113}}$ & 15.0000 $_{\pm \text{ 4.4045}}$ & 22.3000 $_{\pm \text{ 4.5617}}$ & 33.2000 $_{\pm \text{ 5.8787}}$    \\ 
	 & & 4   
	 	& {5.5000} $_{\pm \text{ 2.5787}}$ & 27.2000 $_{\pm \text{ 3.9949}}$ & 19.7000 $_{\pm \text{ 3.0347}}$ & 32.7000 $_{\pm \text{ 4.5617}}$    \\ 
	 & & 5   
	 	& {6.9000} $_{\pm \text{ 1.7578}}$ & 30.7000 $_{\pm \text{ 5.8830}}$ & 21.4000 $_{\pm \text{ 3.6932}}$ & 37.5000 $_{\pm \text{ 6.2809}}$  \\ 
	 \hline 
	
	\multirow{17}{*}{FDR} 
	 & \multirow{5}{*}{sample size} 
	 	 & 200 
	    & {0.2600} $_{\pm \text{ 0.1287}}$ & 0.5053 $_{\pm \text{ 0.0986}}$ & 0.4189 $_{\pm \text{ 0.1527}}$ & 0.8850 $_{\pm \text{ 0.0662}}$    \\ 
	 & & 400   
	    & {0.0356} $_{\pm \text{ 0.0403}}$ & 0.4051 $_{\pm \text{ 0.0698}}$ & 0.3545 $_{\pm \text{ 0.1276}}$ & 0.8321 $_{\pm \text{ 0.0873}}$    \\ 
	 & & 600   
	    & {0.0530} $_{\pm \text{ 0.0420}}$ & 0.2618 $_{\pm \text{ 0.0744}}$ & 0.4034 $_{\pm \text{ 0.1099}}$ & 0.8588 $_{\pm \text{ 0.0904}}$    \\ 
	 & & 800   
	    & {0.0541} $_{\pm \text{ 0.0456}}$ & 0.1125 $_{\pm \text{ 0.0580}}$ & 0.2667 $_{\pm \text{ 0.1351}}$ & 0.8649 $_{\pm \text{0.1121}}$    \\ 
	 & & 1000   
	    & {0.0229} $_{\pm \text{ 0.0250}}$ & 0.1252 $_{\pm \text{ 0.0757}}$ & 0.2675 $_{\pm \text{ 0.0972}}$ & 0.8244 $_{\pm \text{ 0.0978}}$    \\ 
	 \cline{2-7}
	 & \multirow{7}{*}{number of variable} 
	 & 5  
	    & {0.0000} $_{\pm \text{ 0.0000}}$ & 0.0411 $_{\pm \text{ 0.0674}}$ & 0.3075 $_{\pm \text{ 0.2038}}$ & 0.5750 $_{\pm \text{ 0.2583}}$    \\ 
	 & & 10   
	    & {0.0429} $_{\pm \text{ 0.0499}}$ & 0.1261 $_{\pm \text{ 0.0811}}$ & 0.3317 $_{\pm \text{ 0.1872}}$ & 0.8040 $_{\pm \text{ 0.0764}}$    \\ 
	 & & 15   
	    & {0.0530} $_{\pm \text{ 0.0420}}$ & 0.2618 $_{\pm \text{ 0.0744}}$ & 0.4034 $_{\pm \text{ 0.1099}}$ & 0.8588 $_{\pm \text{ 0.0904}}$    \\ 
	 & & 20   
	    & {0.0993} $_{\pm \text{ 0.0573}}$ & 0.2209 $_{\pm \text{ 0.0701}}$ & 0.2793 $_{\pm \text{ 0.1049}}$ & 0.8588 $_{\pm \text{ 0.0320}}$    \\ 
	 & & 25   
	    & {0.0550} $_{\pm \text{ 0.0413}}$ & 0.4765 $_{\pm \text{ 0.0820}}$ & 0.2768 $_{\pm \text{ 0.0925}}$ & 0.8872 $_{\pm \text{ 0.0609}}$    \\ 
	 & & 50   
	    & {0.0590} $_{\pm \text{ 0.0189}}$ & 0.4161 $_{\pm \text{ 0.0427}}$ & 0.2493 $_{\pm \text{ 0.0725}}$ & 0.9668 $_{\pm \text{ 0.0137}}$    \\ 
	 & & 100   
	    & {0.0089} $_{\pm \text{ 0.0096}}$ & 0.2502 $_{\pm \text{ 0.0729}}$ & 0.2608 $_{\pm \text{ 0.0906}}$ & 0.9843 $_{\pm \text{ 0.0066}}$    \\  
	 \cline{2-7}
	 & \multirow{5}{*}{noise variance} 
	 & 1  
	    & {0.0994} $_{\pm \text{0.0543}}$ & 0.1122 $_{\pm \text{ 0.0553}}$ & 0.3814 $_{\pm \text{ 0.2748}}$ & 0.4918 $_{\pm \text{ 0.2577}}$    \\ 
	 & & 2   
	    & {0.0493} $_{\pm \text{ 0.0356}}$ & 0.0791 $_{\pm \text{ 0.0936}}$ & 0.3500 $_{\pm \text{ 0.1125}}$ & 0.7080 $_{\pm \text{ 0.2659}}$    \\ 
	 & & 3   
	    & {0.0530} $_{\pm \text{ 0.0420}}$ & 0.2618 $_{\pm \text{ 0.0744}}$ & 0.4034 $_{\pm \text{ 0.1099}}$ & 0.7462 $_{\pm \text{ 0.2774}}$    \\ 
	 & & 4   
	 	& {0.0766} $_{\pm \text{ 0.0590}}$ & 0.4727 $_{\pm \text{ 0.0668}}$ & 0.3190 $_{\pm \text{ 0.1304}}$ & 0.8322 $_{\pm \text{ 0.1388}}$    \\ 
	 & & 5   
	 	& {0.0456} $_{\pm \text{ 0.0321}}$ & 0.5323 $_{\pm \text{ 0.0994}}$ & 0.3512 $_{\pm \text{ 0.0602}}$ & 0.8491 $_{\pm \text{ 0.0903}}$  \\ 
	 \hline 
	
	\multirow{17}{*}{TPR} 
	 & \multirow{5}{*}{sample size} 
	 	 & 200 
	    & {0.6966} $_{\pm \text{ 0.0822}}$ & 0.5200 $_{\pm \text{ 0.1222}}$ & 0.2166 $_{\pm \text{ 0.1258}}$ & 0.0933 $_{\pm \text{ 0.0466}}$    \\ 
	 & & 400   
	    & {0.7633} $_{\pm \text{ 0.0622}}$ & 0.5833 $_{\pm \text{ 0.1614}}$ & 0.3833 $_{\pm \text{ 0.0933}}$ & 0.1133 $_{\pm \text{ 0.0600}}$    \\ 
	 & & 600   
	    & {0.8800} $_{\pm \text{ 0.0452}}$ & 0.7433 $_{\pm \text{ 0.1135}}$ & 0.3400 $_{\pm \text{ 0.1651}}$ & 0.0666 $_{\pm \text{ 0.0471}}$    \\ 
	 & & 800   
	    & {0.8966} $_{\pm \text{ 0.0314}}$ & 0.7266 $_{\pm \text{ 0.1200}}$ & 0.2933 $_{\pm \text{ 0.1143}}$ & 0.0766 $_{\pm \text{ 0.0633}}$    \\ 
	 & & 1000   
	    & {0.9633} $_{\pm \text{ 0.0314}}$ & 0.7366 $_{\pm \text{ 0.1828}}$ & 0.5766 $_{\pm \text{0.1770}}$ & 0.0800 $_{\pm \text{ 0.0733}}$    \\ 
	 \cline{2-7}
	 & \multirow{7}{*}{number of variable} 
	 & 5  
	    & {0.9700} $_{\pm \text{ 0.0458}}$ & 0.7700 $_{\pm \text{ 0.2238}}$ & 0.3100 $_{\pm \text{ 0.1220}}$ & 0.3199 $_{\pm \text{ 0.2227}}$    \\ 
	 & & 10   
	    & {0.9350} $_{\pm \text{ 0.0502}}$ & 0.7900 $_{\pm \text{ 0.2009}}$ & 0.3050 $_{\pm \text{ 0.2030}}$ & 0.2100 $_{\pm \text{ 0.0888}}$    \\ 
	 & & 15   
	    & {0.8800} $_{\pm \text{ 0.0452}}$ & 0.7433 $_{\pm \text{ 0.1135}}$ & 0.3400 $_{\pm \text{ 0.1651}}$ & 0.0666 $_{\pm \text{ 0.0471}}$    \\ 
	 & & 20   
	    & {0.9000} $_{\pm \text{ 0.0295}}$ & 0.5150 $_{\pm \text{ 0.1135}}$ & 0.4150 $_{\pm \text{ 0.0653}}$ & 0.1300 $_{\pm \text{ 0.0484}}$    \\ 
	 & & 25   
	    & {0.8820} $_{\pm \text{ 0.0501}}$ & 0.5960 $_{\pm \text{ 0.1123}}$ & 0.4360 $_{\pm \text{ 0.0897}}$ & 0.0880 $_{\pm \text{ 0.0499}}$    \\ 
	 & & 50   
	    & {0.6359} $_{\pm \text{ 0.0427}}$ & 0.5480 $_{\pm \text{ 0.0460}}$ & 0.4320 $_{\pm \text{ 0.0720}}$ & 0.0270 $_{\pm \text{ 0.0126}}$    \\ 
	 & & 100   
	    & {0.4350} $_{\pm \text{ 0.0336}}$ & 0.3065 $_{\pm \text{ 0.0492}}$ & 0.3400 $_{\pm \text{ 0.0999}}$ & 0.0120 $_{\pm \text{ 0.0060}}$    \\  
	 \cline{2-7}
	 & \multirow{5}{*}{noise variance} 
	 & 1  
	    & {0.9100} $_{\pm \text{ 0.0448}}$ & 0.9266 $_{\pm \text{ 0.0442}}$ & 0.2966 $_{\pm \text{ 0.2243}}$ & 0.2366 $_{\pm \text{ 0.1277}}$    \\ 
	 & & 2   
	    & {0.8966} $_{\pm \text{ 0.0504}}$ & 0.6666 $_{\pm \text{ 0.2708}}$ & 0.4100 $_{\pm \text{ 0.0683}}$ & 0.0933 $_{\pm \text{ 0.0628}}$    \\ 
	 & & 3   
	    & {0.8800} $_{\pm \text{ 0.0452}}$ & 0.7433 $_{\pm \text{ 0.1135}}$ & 0.3400 $_{\pm \text{ 0.1651}}$ & 0.0733 $_{\pm \text{ 0.0592}}$    \\ 
	 & & 4   
	 	& {0.8900} $_{\pm \text{ 0.0395}}$ & 0.6400 $_{\pm \text{ 0.1768}}$ & 0.3766 $_{\pm \text{ 0.1075}}$ & 0.0500 $_{\pm \text{ 0.0401}}$    \\ 
	 & & 5   
	 	& {0.8100} $_{\pm \text{ 0.0578}}$ & 0.4900 $_{\pm \text{0.1350}}$ & 0.3533 $_{\pm \text{ 0.1284}}$ & 0.0800 $_{\pm \text{ 0.0600}}$  \\ 
	 \hline

\end{tabular}

\end{center}
\end{table*}

}

\bibliographystyle{IEEEtran}
\bibliography{IEEEabrv,ref}



\begin{IEEEbiography}
[{\includegraphics[width=1in, height=1.25in, clip, keepaspectratio]{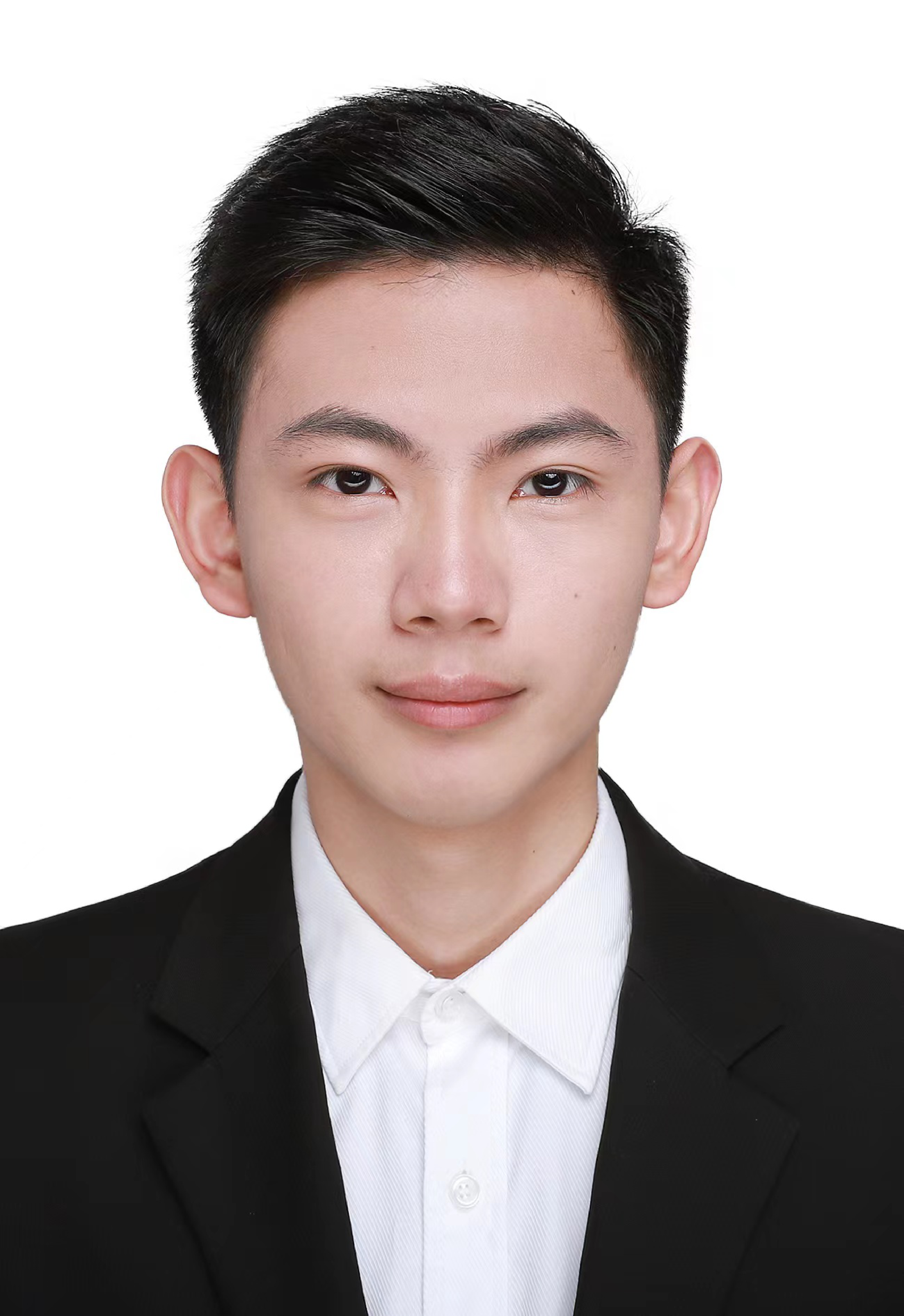}}]{Weilin Chen} received the B.S. degree in software engineering from Guangdong University of Technology, Guangzhou, China, in 2020, where he is currently pursuing the Ph.D. degree with the School of Computer. His current research interests include causal inference and machine learning.
\end{IEEEbiography}

\begin{IEEEbiography}
[{\includegraphics[width=1in, height=1.25in, clip, keepaspectratio]{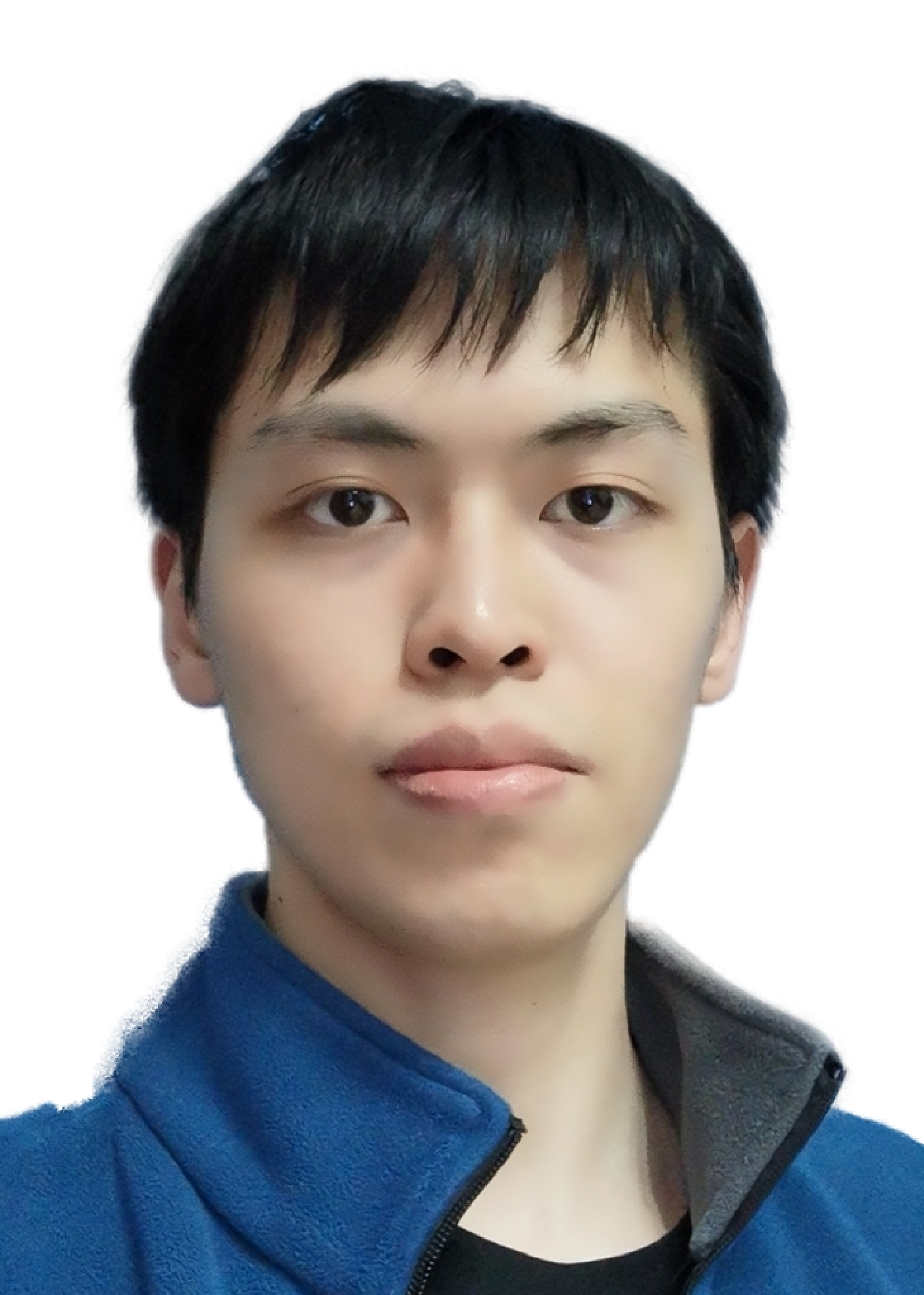}}]{Jie Qiao} received the Ph.D. from Guangdong University of Technology, school of computer science, in 2021. He is currently a postdoctoral researcher in Guangdong University of Technology. His research interests include causal discovery and causality-inspired machine learning.
\end{IEEEbiography}

\begin{IEEEbiography}[{\includegraphics[width=1in, height=1.25in, clip, keepaspectratio]{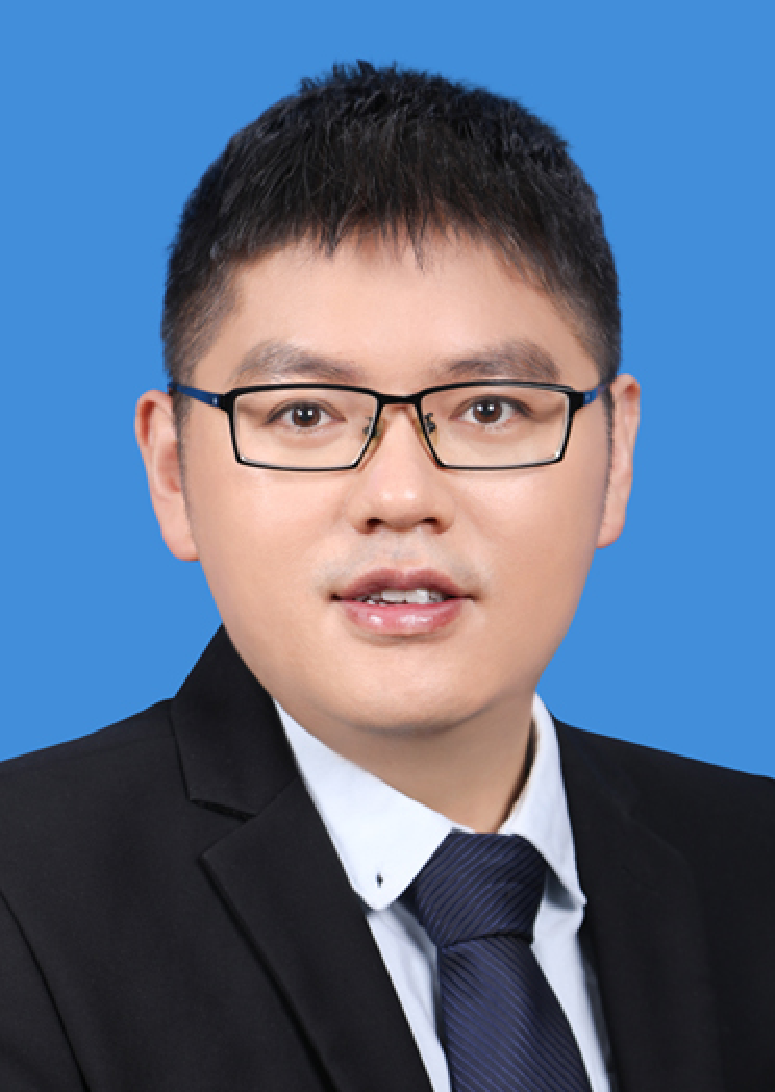}}]{Ruichu Cai} (M'17) is currently a professor in the school of computer science and the director of the data mining and information retrieval laboratory, Guangdong University of Technology. He received his B.S. degree in applied mathematics and Ph.D. degree in computer science from South China University of Technology in 2005 and 2010, respectively. 
 
His research interests cover various topics, including causality, deep learning, and their applications. He was a recipient of the National Science Fund for Excellent Young Scholars, the Natural Science Award of Guangdong, and so on awards. He has served as the action editor of Neural Networks, the area chair of ICML 2022-2024, NeurIPS 2022-2024, ICLR 2024 and UAI 2022-2024, the senior PC of AAAI 2019-2022, IJCAI 2019-2022, and so on. He is now a senior member of CCF and IEEE.

\end{IEEEbiography}
\vspace{-8ex}

\begin{IEEEbiography}[{\includegraphics[width=1in, height=1.25in, clip, keepaspectratio]{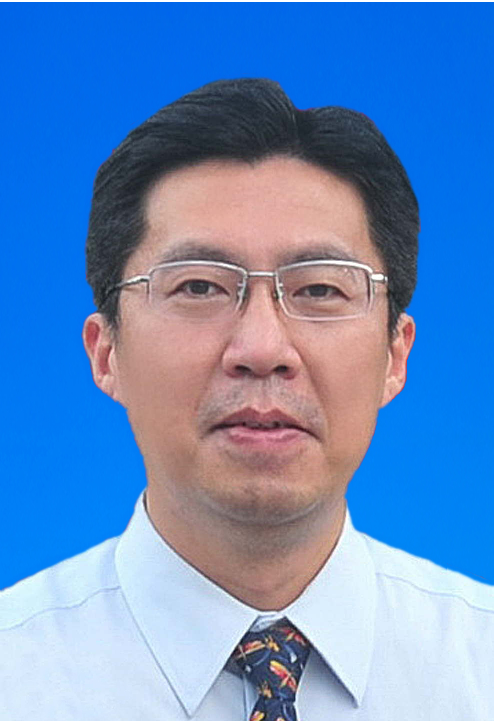}}]{Zhifeng Hao} received his B.S. degree in Mathematics from the Sun Yat-Sen University in 1990, and his Ph.D. degree in Mathematics from Nanjing University in 1995. He is currently a Professor in the School of Computer, Guangdong University of Technology, and College of Science, Shantou University. 
	
His research interests involve various aspects of Algebra, Machine Learning, Data Mining, Evolutionary Algorithms.
\end{IEEEbiography}

\vfill

\end{document}